\documentclass{article}

\usepackage[final]{neurips_2021}
\PassOptionsToPackage{numbers, compress}{natbib}

\usepackage[utf8]{inputenc} %
\usepackage[T1]{fontenc}    %
\usepackage{hyperref}       %
\usepackage{url}            %
\usepackage{booktabs}       %
\usepackage{amsfonts}       %
\usepackage{nicefrac}       %
\usepackage{microtype}      %

\usepackage[dvipsnames]{xcolor}
\usepackage{amsmath,bm,amssymb,amsthm}
\usepackage{mathtools}
\usepackage{multirow}
\usepackage{url}
\usepackage[nameinlink]{cleveref}
\usepackage{enumitem}  %
\usepackage{subcaption}
\usepackage{wrapfig}
\usepackage{thmtools,thm-restate}

\usepackage{definitions}

\title{
    Understanding Negative Samples
    \\ in Instance Discriminative Self-supervised Representation Learning
}

\author{%
    Kento Nozawa \\
    The University of Tokyo \& RIKEN AIP \\
    \texttt{nzw@g.ecc.u-tokyo.ac.jp} \\
    \And
    Issei Sato \\
    The University of Tokyo \\
    \texttt{sato@g.ecc.u-tokyo.ac.jp} \\
}

\begin{document}

\maketitle

\begin{abstract}
    Instance discriminative self-supervised representation learning has been attracted attention thanks to its unsupervised nature and informative feature representation for downstream tasks.
    In practice, it commonly uses a larger number of negative samples than the number of supervised classes.
    However, there is an inconsistency in the existing analysis;
    theoretically, a large number of negative samples degrade classification performance on a downstream supervised task, while empirically, they improve the performance.
    We provide a novel framework to analyze this empirical result regarding negative samples using the coupon collector's problem.
    Our bound can implicitly incorporate the supervised loss of the downstream task in the self-supervised loss by increasing the number of negative samples.
    We confirm that our proposed analysis holds on real-world benchmark datasets.
\end{abstract}

\section{Introduction}
\label{sec:introduction}

Self-supervised representation learning is a popular class of unsupervised representation learning algorithms in the domains of vision~\citep{Bachman2019NeurIPS,Chen2020ICML,He2020CVPR,Caron2020NeurIPS,Grill2020NeurIPS,Chen2021CVPR} and language~\citep{Mikolov2013NeurIPS,Devlin2019NAACL,Brown2020NeurIPS}.
Generally, it trains a feature extractor by solving a pretext task constructed on a large unlabeled dataset.
The learned feature extractor yields generic feature representations for other machine learning tasks such as classification.
Recent self-supervised representation learning algorithms help a linear classifier to attain classification accuracy comparable to a supervised method from scratch, especially in a few amount of labeled data regime~\citep{Newell2020CVPR,Henaff2020ICML,Chen2020NeurIPS}.
For example, \texttt{SwAV}~\citep{Caron2020NeurIPS} with ResNet-50 has a top-1 validation accuracy of $75.3\%$ on the ImageNet-1K classification~\citep{Deng2009CVPR} compared with $76.5\%$ by using the fully supervised method.

InfoNCE~\citep[Eq. 4]{Oord2018arXiv} or its modification is a de facto standard loss function used in many state-of-the-art self-supervised methods~\citep{Logeswaran2018ICLR,Bachman2019NeurIPS,He2020CVPR,Chen2020ICML,Henaff2020ICML,Caron2020NeurIPS}.
Intuitively, the minimization of InfoNCE can be viewed as the minimization of cross-entropy loss on $K+1$ instance-wise classification, where $K$ is the number of negative samples.
Despite the empirical success of self-supervised learning, we still do not understand why the self-supervised learning algorithms with InfoNCE perform well for downstream tasks.

\citet{Arora2019ICML} propose the first theoretical framework for contrastive unsupervised representation learning (CURL).
However, there exists a gap between the theoretical analysis and empirical observation as in~\cref{fig:comparison_bound_curve}.
Precisely, we expect that a large number of negative samples degrade a supervised loss on a downstream task from the analysis by~\citet{Arora2019ICML}.
In practice, however, a large number of negative samples are commonly used in self-supervised representation learning algorithms~\citep{He2020CVPR,Chen2020ICML}.

\paragraph{Contributions.}
We show difficulty to explain why large negative samples empirically improve supervised accuracy on the downstream task from the CURL framework when we use learned representations as feature vectors for the supervised classification in~\cref{sec:extended-curl}.
To fill the gap, we propose a novel lower bound to theoretically explain this empirical observation regarding negative samples using the coupon collector's problem in~\cref{sec:our-analysis}.

\section{InfoNCE-based Self-supervised Representations Learning}
\label{sec:preliminaries}

\begin{wrapfigure}[25]{R}{0.5\textwidth}
    \centering
    \vspace{-\intextsep}
    \includegraphics[width=0.5\textwidth]{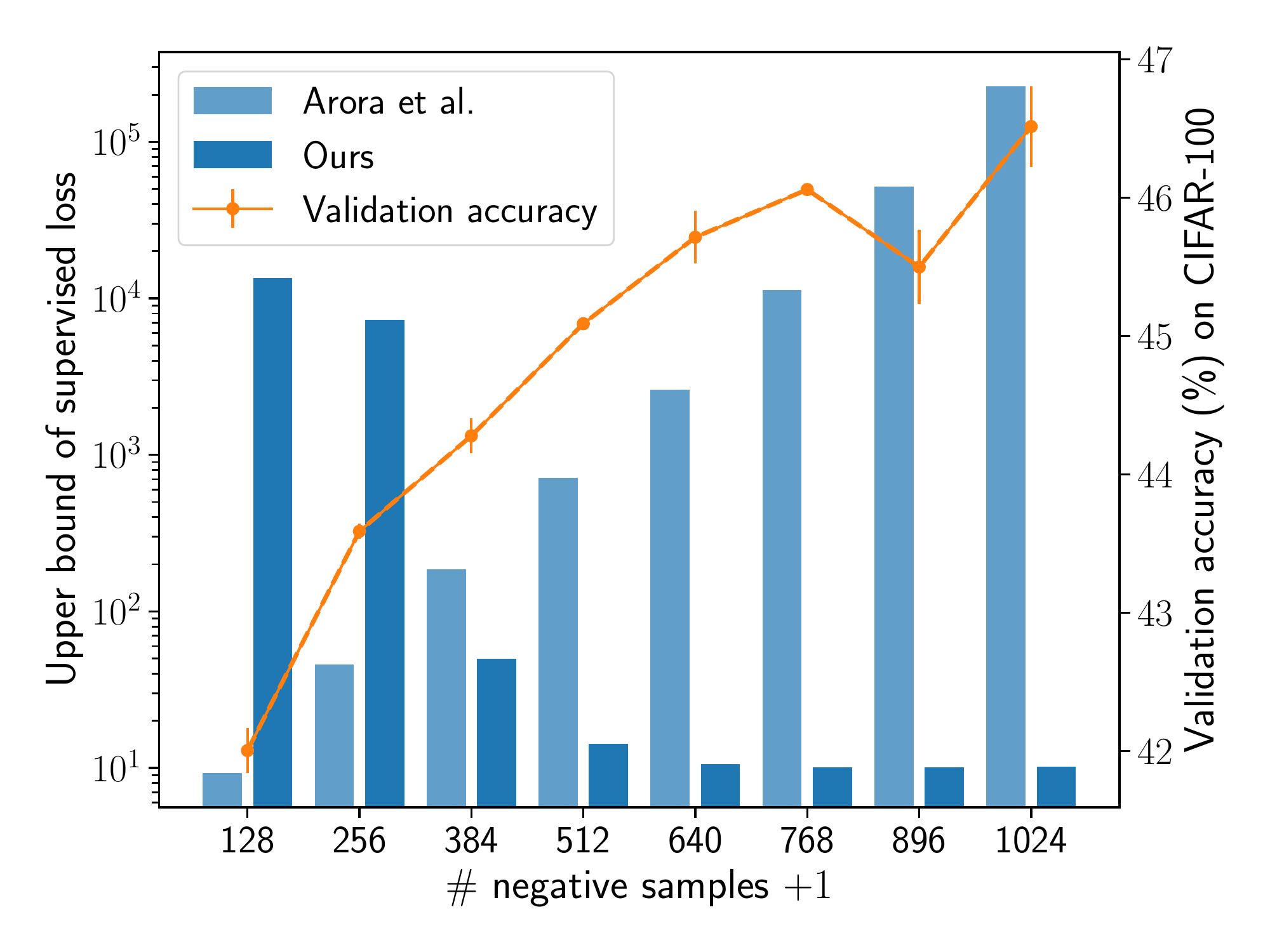}
    \caption{
        Upper bounds of supervised loss and validation accuracy on CIFAR-100.
        By increasing the number of negative samples, validation accuracy tends to improve.
        \textbf{Left bars}: The existing contrastive unsupervised representation learning bound~\eqref{eq:decompose-sup} also increases when the number of negative samples increases because the bound of supervised loss explodes due to a collision term that is not related to classification loss (see \cref{eq:decompose-sup} for the definition).
        \textbf{Right bars with $^\bigstar$}: On the other hand, the proposed counterpart~\eqref{eq:proposed-bound} does not explode.
        \cref{sec:experiments} describes the details of this experiment.
    }
    \label{fig:comparison_bound_curve}
\end{wrapfigure}

We focus on~\citet{Chen2020ICML}'s self-supervised representation learning formulation, namely, \texttt{SimCLR}.%
\footnote{In fact, our theoretical analysis is valid with asymmetric feature extractors such as \texttt{MoCo}~\citep{He2020CVPR}, where the positive and negative features do not come from feature extractor $\fbf$.}
Let $\Xcal$ be an input space, e.g., $\Xcal \subset \Rbb^{\mathrm{channel} \times \mathrm{width} \times \mathrm{height}}$ for color images.
We can only access a unlabeled training dataset $\{ \xbf_i \}_{i=1}^N$, where input $\xbf \in \Xcal$.
\texttt{SimCLR} learns feature extractor $\fbf: \Xcal \rightarrow \Rbb^h$ modeled by neural networks on the dataset, where $h$ is the dimensionality of feature representation.
Let $\zbf = \fbf(\abf(\xbf))$ that is a feature representation of $\xbf$ after applying data augmentation $\abf: \Xcal \rightarrow \Xcal$.
Data augmentation is a pre-defined stochastic function such as a composition of the horizontal flipping and cropping.
Note that the output of $\fbf$ is normalized by its L2 norm.
Let $\zbf^{+} = \fbf(\abf^{+}(\xbf))$ that is a positive feature representation created from $\xbf$ with different data augmentation $\abf^{+}(\cdot)$.

\texttt{SimCLR} minimizes the following InfoNCE-based loss with $K$ negative samples for each pair of $(\zbf, \zbf^{+})$:
\begin{align}
    \ell_{\mathrm{Info}}(\zbf, \Zbf)
    \coloneqq -
    \ln \frac{
        \exp \left(\zbf \cdot \zbf^{+} / t \right)
        }{
            \sum_{\zbf_k \in \Zbf}
            \exp \left(\zbf \cdot \zbf_k / t \right)}
    ,
    \label{eq:NT-Xent}
\end{align}
where $\Zbf = \{\zbf^{+}, \zbf^-_{1}, \ldots, \zbf^-_{K}\}$ that is a set of positive representation $\zbf^{+}$ and $K$ negative representations $\{\zbf^-_{1} \ldots, \zbf^-_{K}\}$ created from other samples, $(. \cdot .)$ is an inner product of two representations, and $t \in \Rbb_+$ is a temperature parameter.%
\footnote{We perform the analysis with $t=1$ for simplicity, but our analysis holds with any temperature.}
Intuitively, this loss function approximates an instance-wise classification loss by using $K$ random negative samples~\citep{Wu2018CVPR}.
From the definition of~\cref{eq:NT-Xent}, we expect that $\fbf$ learns an invariant encoder with respect to data augmentation $\abf$.
After minimizing~\cref{eq:NT-Xent}, $\fbf$ works as a feature extractor for downstream tasks such as classification.

\section{Extension of Contrastive Unsupervised Representation Learning Framework}
\label{sec:extended-curl}
We focus on the direct relationship between the self-supervised loss and a supervised loss to understand the role of the number of negative samples $K$.
For a similar problem, the CURL framework~\citep{Arora2019ICML} shows that the averaged supervised loss is bounded by the contrastive unsupervised loss.
Thus, we extend the analysis of CURL to the self-supervised representation learning in~\cref{sec:preliminaries} and point out that we have difficulty explaining the empirical observation as in~\cref{fig:comparison_bound_curve} from the existing analysis.
\cref{tab:notations} in~\cref{sec:appenxi-notation} summarizes notations used in this paper for convenience.

\subsection{CURL Formulation for \texttt{SimCLR}}
By following the CURL analysis~\citep{Arora2019ICML},
we introduce the learning process in two steps: \textit{self-supervised representation learning step} and \textit{supervised learning step} (see~\cref{sec:self-sup-step} and~\cref{sec:supervised-learning}, respectively).

\subsubsection{Self-supervised representation learning step}
\label{sec:self-sup-step}
The purpose of the self-supervised learning step is to learn a feature extractor $\fbf$ on an unlabeled dataset.
During this step, we can only access input samples from $\Xcal$ without any relationship between samples, unlike metric learning~\citep{Kulis2012} or similar unlabeled learning~\citep{Bao2018ICML}.

We formulate the data generation process for~\cref{eq:NT-Xent}.
The key idea of the CURL analysis is the existence of latent classes $\Ccal$ that are associated with supervised classes.%
\footnote{
    Technically, we do not need $c$ to draw $\xbf$ for self-supervised representation learning. However, we explicitly include it in the data generation process to understand the relationship with a supervised task in~\cref{sec:simple-analysis}.
    In the main analysis, we assume that the supervised classes in the downstream task are subset of $\Ccal$; however, different relationship between supervised and latent classes is discussed in~\cref{sec:details-latent-class}.
}
Let $\rho$ be a probability distribution over $\Ccal$,
and let $\xbf$ be an input sample drawn from a data distribution $\Dc$ conditioned on a latent class $c \in \Ccal$.
We draw two data augmentations $(\abf, \abf^{+})$ from the distribution of data augmentation $\Acal$ and apply them independently to the sample.
As a result, we have two augmented samples $(\abf(\xbf), \abf^{+}(\xbf))$ and call $\abf^{+}(\xbf)$ as a positive sample of $\abf(\xbf)$.
Similarly, we draw $K$ negative samples from $\mathcal{D}_{c^-_k}$ for each $c^-_k \in \{c^-_1, \ldots, c^-_K\} \sim \rho^{K}$
and $K$ data augmentations from $\Acal$, and then we apply them to negative samples.
\cref{definition:data-generation} summarizes the data generation process.
Note that we suppose data augmentation $\abf \sim \Acal$ does not change the latent class of $\xbf \in \Xcal$.
\begin{definition}[Data Generation Process in Self-supervised Representation Learning Step]
    \;
        \\
        1. Draw latent classes: $c, \{c_k^-\}_{k=1}^K \sim \rho^{K+1}$;
        \hspace{5mm} 2. Draw input sample: $\xbf \sim \Dc$; \\
        3. Draw data augmentations: $(\abf, \abf^{+}) \sim \Acal^{2}$;
        \hspace{5mm} 4. Apply data augmentations: $\abf(\xbf), \abf^{+}(\xbf)$; \\
        5. Draw negative samples: $\{\xbf^-_k \}_{k=1}^K \sim \Dcal_{c^-_k}^{K}$;
        \hspace{4mm} 6. Draw data augmentations: $\{ \abf^-_k \}_{k=1}^K \sim \Acal^{K}$; \\
        7. Apply data augmentations: $\{ \abf^-_k(\xbf^-_k) \}_{k=1}^K$.
    \label{definition:data-generation}
\end{definition}
Note that the original data generation process of CURL samples a positive sample of $\xbf$ from $\mathcal{D}_c$ independently and does not use any data augmentations.

From the data generation process and InfoNCE loss~\eqref{eq:NT-Xent},
we give the following formal definition of self-supervised loss.
\begin{definition}[Expected Self-supervised Loss]
    \begin{align}
        \Lin(\fbf)
        \coloneqq
        \Ebb_{
            \substack{
                c, \{c_k^-\}_{k=1}^K
                \sim \rho^{K+1}
            }
        }
        \Ebb_{
            \substack{
                \xbf \sim \Dc
                \\
                (\abf, \abf^{+}) \sim \Acal^{2}
            }
        }
        \Ebb_{
            \substack{
                \{ \xbf^-_k \sim \Dcal_{c^-_k} \}_{k=1}^K
                \\
                \{ \abf_k^- \}_{k=1}^K \sim \Acal^{K}
            }
        }
        \ell_{\mathrm{Info}}(\zbf, \Zbf),
        \label{eq:test-self-sup-loss}
    \end{align}
    \label{definition:self-supervised-loss}
    where recall that $\zbf = \fbf(\abf(\xbf))$ and $\Zbf = \{ \fbf(\abf^{+}(\xbf)), \fbf(\abf_1^{-}(\xbf^{-}_1)), \ldots, \fbf(\abf_K^{-}(\xbf^{-}_K))\}$.
\end{definition}
Since we cannot directly minimize~\cref{eq:test-self-sup-loss},
we minimize the empirical counterpart, $\widehat{L}_{\mathrm{Info}}(\fbf)$, by sampling from the training dataset.
After minimizing $\widehat{L}_{\mathrm{Info}}(\fbf)$, learned $\widehat{\fbf}$ works as a feature extractor in the supervised learning step.

\subsubsection{Supervised learning step}
\label{sec:supervised-learning}

At the supervised learning step, we can observe label $y \in \Ycal = \{1, \ldots, Y \}$ for each $\xbf$ that is used in the self-supervised learning step.
Let $\Scal$ be a supervised data distribution over $\Xcal \times \Ycal$,
and $\gbf \circ \widehat{\fbf}: \Xcal \rightarrow \Rbb^Y$ be a classifier, where $\gbf: \Rbb^{h} \rightarrow \Rbb^{Y}$ and $\widehat{\fbf}$ is the frozen feature extractor.
Given the supervised data distribution and classifier, our goal is to minimize the following supervised loss:
\begin{align}
    {L}_\mathrm{sup}(\gbf \circ \widehat{\fbf})
    \coloneqq
    \Ebb_{\substack{
        \xbf, y \sim \Scal \\
        \abf \sim \Acal
    }}
    - \ln \frac{\exp \left(\gbf_y(\widehat{\fbf}(\abf(\xbf))) \right) }{
        \sum_{j \in \Ycal}
        \exp \left( \gbf_j(\widehat{\fbf}(\abf(\xbf))) \right) }.
    \label{eq:supervised-loss}
\end{align}

By following the CURL analysis, we introduce a mean classifier as a simple instance of $\gbf$ because its loss is an upper bound of~\cref{eq:supervised-loss}.%
\footnote{Concrete inequality is found in \cref{sec:mean-classifier-and-linear-classifier}.}
The mean classifier is the linear classifier whose weight of label $y$ is computed by averaging representations: $\mubf_y = \Ebb_{\xbf \sim \Dy} \Ebb_{\abf \sim \Acal} {\fbf}(\abf(\xbf))$, where $\Dy$ is a data distribution conditioned on the supervise label $y$.
We introduce the definition of the mean classifier's supervised loss as follows:
\begin{definition}[Mean Classifier's Supervised Loss]
    \begin{align}
        \Lsup^\mu(\widehat{\fbf}) \coloneqq
        \Ebb_{\substack{
            \xbf, y \sim \Scal \\
            \abf \sim \Acal
        }}
        -
        \ln \frac{\exp \left( \widehat{\fbf}(\abf(\xbf)) \cdot \boldsymbol{\mu}_{y} \right)
        }{
                \sum_{j \in \Ycal}
                \exp \left( \widehat{\fbf}(\abf(\xbf)) \cdot \boldsymbol{\mu}_j \right)
        }.
    \label{eq:mean-classifier-loss}
    \end{align}
\end{definition}
We also introduce a sub-class loss function.%
\footnote{\citet{Arora2019ICML} refer to this loss function as \textit{averaged supervised loss}.}
Let $\Ycal_\mathrm{sub}$ be a subset of $\Ycal$ and $\Scal_{\mathrm{sub}}$ be a data distribution over $\Xcal \times \Ycal_\mathrm{sub}$,
then we define the sub-class losses of classifier $\gbf$ and mean classifier:
\begin{definition}[Supervised Sub-class Losses of Classifier $\gbf$ and Mean Classifier with $\widehat{\fbf}$]
    \begin{align}
        {L}_\mathrm{sub}(\gbf \circ \widehat{\fbf}, \Ycal_\mathrm{sub})
        & \coloneqq
        \Ebb_{\substack{
            \xbf, y \sim \Scal_{\mathrm{sub}} \\
            \abf \sim \Acal
        }}
        - \ln \frac{\exp \left( \gbf_y(\widehat{\fbf}(\abf(\xbf))) \right) }{
            \sum_{j \in \Ycal_\mathrm{sub}}
            \left( \gbf_j(\widehat{\fbf}(\abf(\xbf))) \right)
        },
        \label{eq:g-sub-class-loss}
        \\
        {L}^{\mu}_\mathrm{sub}(\widehat{\fbf}, \Ycal_\mathrm{sub})
        & \coloneqq
        \Ebb_{\substack{
            \xbf, y \sim \Scal_{\mathrm{sub}} \\
            \abf \sim \Acal
        }}
        - \ln \frac{
                \exp \left( \widehat{\fbf} (\abf(\xbf)) \cdot \mubf_{y} \right)
            }{
                \sum_{j \in \Ycal_\mathrm{sub}}
                \exp \left( \widehat{\fbf} (\abf(\xbf)) \cdot \mubf_{j} \right)
            }.
        \label{eq:sub-class-loss}
    \end{align}
    \label{definition:sub-class-mean-classifier}
\end{definition}

The purpose of unsupervised representation learning~\citep{Bengio2013IEEE} is to learn generic feature representation rather than improve the accuracy of the classifier on the same dataset.
However, we believe that such a feature extractor tends to transfer well to another task.
Indeed, \citet{Kornblith2019CVPR} empirically show a strong correlation between ImageNet's accuracy and transfer accuracy.

\subsection{Theoretical Analysis based on CURL}
\label{sec:simple-analysis}
We show that InfoNCE loss~\eqref{eq:test-self-sup-loss} is an upper bound of the expected sub-class loss of the mean classifier~\eqref{eq:sub-class-loss}.

\paragraph{Step 1. Introduce a lower bound}
We denote $\mubf(\xbf) = \Ebb_{\abf \sim \Acal } \fbf(\abf(\xbf))$ and derive a lower bound of unsupervised loss $\Lin(\fbf)$.
\begin{align}
    \Lin(\fbf)
    & \geq
    \Ebb_{
        c, \{c_k^-\}_{k=1}^K \sim \rho^{K+1}
    }
    \hspace{1em}
    \Ebb_{\substack{
        \xbf \sim \Dc \\
        \abf \sim \Acal \\
        }
    }
    \Ebb_{
        \{ \xbf^-_k \sim \Dcal_{c^-_k} \}_{k=1}^K
    }
    \ell_{\mathrm{Info}} \left(
        \fbf(\abf(\xbf)),
        \left\{
            \mubf(\xbf),
            \mubf(\xbf_{1}^-),
            \ldots,
            \mubf(\xbf_{K}^-)
        \right\}
    \right)
    \nonumber
    \\
    & \geq
    \Ebb_{
        c, \{c_k^-\}_{k=1}^K \sim \rho^{K+1}
    }
    \hspace{1em}
    \Ebb_{
        \substack{
            \xbf \sim \Dc \\
            \abf \sim \Acal
        }
    }
    \ell_{\mathrm{Info}} \left(
        \fbf(\abf(\xbf)),
        \left\{
            \mubf(\xbf),
            \mubf_{c^-_1},
            \ldots,
            \mubf_{c^-_K}
        \right\}
    \right)
    \nonumber
    \\
    & \geq
    \Ebb_{
        c, \{c_k^-\}_{k=1}^K \sim \rho^{K+1}
    }
    \hspace{1em}
    \Ebb_{
        \substack{
            \xbf \sim \Dc \\
            \abf \sim \Acal
        }
    }
    \ell_{\mathrm{Info}}\left(
        \fbf(\abf(\xbf)),
        \left\{
            \mubf_{c},
            \mubf_{c^-_1},
            \ldots,
            \mubf_{c^-_K}
        \right\}
    \right)
    +
    d(\fbf),
    \label{eq:mean-supervised}
    \\
    & \text{where }
    d(\fbf) = \frac{1}{t}
    \Ebb_{\substack{c \sim \rho}}
    \Ebb_{\substack{\xbf \sim \Dc }}
    \left \{
        \Ebb_{\abf \sim \Acal }
        \left[ \fbf(\abf(\xbf)) \right]
        \cdot \left[
            \Ebb_{\substack{\xbf^+ \sim \Dc \\ \abf_1^+ \sim \Acal }}
            \left[ \fbf(\abf_1^+(\xbf^+)) \right]
            -
            \Ebb_{\abf_2^+ \sim \Acal}
            \left[ \fbf(\abf_2^+(\xbf)) \right]
        \right]
    \right\}.
    \nonumber
\end{align}
The first and second inequalities are done by using Jensen's inequality for convex function.
The proof of the third inequality is shown in~\cref{sec:proof-mean-sup}.
It is worth noting that positive and negative features can be extracted from another feature encoder or memory back~\citet{He2020CVPR}.

\begin{remark}[Effect of gap term $d(\fbf)$]
    We confirm that $d(\fbf)$ is an almost constant among different $K$ in practice (see~\cref{tab:bounds}).
    Therefore we focus on the first term in~\cref{eq:mean-supervised} in the following analysis.
\end{remark}

\paragraph{Step 2. Decomposition into the expected sub-class loss}
We convert the first term of~\cref{eq:mean-supervised} into the expected sub-class loss explicitly.
By following~\citet[Theorem B.1]{Arora2019ICML}, we introduce
collision probability: $\tau_K = \Pbb(\mathrm{Col}(c, \{c^-_k\}_{k=1}^{K} ) \neq 0)$,
where $\mathrm{Col}(c, \{c^-_k\}_{k=1}^{K}) = \sum_{k=1}^K \Ibb[c = c^-_k]$ and $\Ibb[\cdot]$ is the indicator function.
We omit the arguments of $\mathrm{Col}$ for simplicity.
Let $\Ccal_{\mathrm{sub}}(\{c, c^-_1, \ldots, c^-_K\})$ be a function to remove duplicated latent classes given latent classes.
We omit the arguments of $\Ccal_{\mathrm{sub}}$ as well.
The result is our extension of~\citet[Lemma 4.3]{Arora2019ICML} for self-supervised representation learning as the following proposition:
\begin{restatable}[CURL Lower Bound of Self-supervised Loss]{proposition}{CURLLowerBound}
    \label{proposition:curl-lower-bound}
    For all feature extractor $\fbf$,
    \begin{align}
        L_{\mathrm{Info}}(\fbf)
        \geq (1-\tau_K)
        &
        \Ebb_{\substack{c, \{c^-_k \}_{k=1}^K \sim \rho^{K+1} }}
        [
            \underbrace{
                L^{\mu}_{\mathrm{sub}}(\fbf, \Ccal_{\mathrm{sub}} )
            }_{\text{sub-class loss}}
            \mid \mathrm{Col} = 0
        ]
        \nonumber
        \\
        +
        \tau_K
        &
        \Ebb_{\substack{c, \{c^-_k \}_{k=1}^K \sim \rho^{K+1} }}
        [
            \underbrace{
                \ln (\mathrm{Col} + 1)
            }_{\text{collision}}
        \mid \mathrm{Col} \neq 0 ]
        + d(\fbf).
        \label{eq:decompose-sup}
    \end{align}
\end{restatable}
The proof is found in~\cref{sec:proof-proposition}.

\citet{Arora2019ICML} also show Rademacher complexity-based generalization error bound for the mean classifier.
Since we focus on the behavior of the self-supervised loss rather than a generalization error bound when $K$ increases, we do not perform further analysis as was done in~\citet{Arora2019ICML}.
Nevertheless, we believe that constructing a generalization error bound by following either~\citet[Theorem 4.1]{Arora2019ICML} or~\citet[Theorem 7]{Nozawa2020UAI} is worth interesting future work.

\subsection{Limitations of~\cref{eq:decompose-sup}}
\label{sec:limitation-curl}
The bound~\eqref{eq:decompose-sup} tells us that $K$ gives a trade-off between the sub-class loss and collision term via $\tau_K$.
Recall that our final goal is to minimize the supervised loss~\eqref{eq:mean-classifier-loss} rather than expected sub-class loss~\eqref{eq:sub-class-loss}.
To incorporate the supervised loss~\eqref{eq:mean-classifier-loss} into the lower bound~\eqref{eq:decompose-sup}, we need $K$ large enough to satisfy $\Ccal_{\mathrm{sub}} \supseteq \Ycal$.
However, such a large $K$ makes the lower bound meaningless.

We can easily observe that the lower bound~\eqref{eq:decompose-sup} converges to the collision term by increasing $K$ since the collision probability $\tau_K$ converges to $1$.
As a result, the sub-class loss rarely contributes to the lower bound.
Let us consider the bound on CIFAR-10, where the number of supervised classes is $10$, and latent classes are the same as the supervised classes.
When $\tau_{K=32} \approx 0.967$, i.e., the only $3.3\%$ training samples contribute to the expected sub-class loss,
the others fall into the collision term.
Indeed, \citet{Arora2019ICML} show that small latent classes or large negative samples degrade classification performance of the expected sub-class classification task.
However, even much larger negative samples, $K+1=512$, yield the best performance of a linear classifier with self-supervised representation on CIFAR-10~\citep[B.9]{Chen2020ICML}.

\section{Proposed Lower Bound for Instance-wise Self-supervised Representation Learning}
\label{sec:our-analysis}

To fill the gap between the theoretical bound~\eqref{eq:decompose-sup} and empirical observation from recent work,
we propose another lower bound for self-supervised representation learning.
The key idea of our bound is to replace $\tau$ with a different probability since $\tau$ can quickly increase depending on $K$.
The idea is motivated by focusing on the supervised loss rather than the expected sub-class loss.

\subsection{Proposed Lower Bound}
\label{sec:upsilon-decompose}
Let $\upsilon_{K}$ be a probability that sampled $K$ latent classes contain all latent classes: $\{c_{k} \}_{k=1}^K \supseteq \Ccal$.
This probability appears in the coupon collector's problem of probability theory (e.g., \citet[Example 2.2.7]{Durrett2019Book}).
If the latent class probability is uniform: $\forall c, \rho(c) = 1 / |\Ccal|$,
then we can calculate the probability explicitly as follows.
\begin{definition}[Probability to Draw All Latent Classes]
    Assume that $\rho$ is a uniform distribution over latent classes $\Ccal$.
    The probability that $K$ latent classes drawn from $\rho$ contain all latent classes is defined as
    \begin{align}
        \upsilon_K &\coloneqq
        \sum^K_{n=1} \sum_{m=0}^{|\Ccal|-1} {|\Ccal|-1 \choose m} (-1)^m \left(1 - \frac{m+1}{|\Ccal|} \right)^{n-1},
        \label{eq:coupon-probability}
    \end{align}
    where the first summation is a probability that $n$ drawn latent samples contain all latent classes~\citep[Eq. 2]{Nakata2006}.%
    \footnote{We show expected $K+1$ to draw all supervised labels for ImageNet-1K and all used datasets in~\cref{sec:expeted_number}.}
\end{definition}

By replacing $\tau$ with $\upsilon$, we obtain our lower bound of InfoNCE loss:
\begin{restatable}[Proposed Lower Bound of Self-supervised Loss]{theorem}{ProposedBound}
    \label{theorem:proposed-lower-bound}
    For all feature extractor $\fbf$,
    \begin{align}
        \Lin(\fbf)
        \geq
        \frac{1}{2}
        \Big\{
        \upsilon_{K+1}
        &
        \Ebb_{\substack{c, \{c^-_k\}_{k=1}^K \sim \rho^{K+1} }}
            [
            \underbrace{
                L_{\mathrm{sub}}^{\mu}(\fbf, \Ccal)
            }_{\text{sup. loss}}
            \mid \Ccal_{\mathrm{sub}} = \Ccal
            ]
        \nonumber
        \\
        +
        (1-\upsilon_{K+1})
        &
        \Ebb_{\substack{c, \{c^-_k\}_{k=1}^K \sim \rho^{K+1} }}
        [
        \underbrace{
            L^{\mu}_{\mathrm{sub}}( \fbf, \Ccal_{\mathrm{sub}} )
        }_{\text{sub-class loss}
        }
        \mid \Ccal_{\mathrm{sub}} \neq \Ccal
        ]
        \nonumber
        \\
        +
        &
        \Ebb_{\substack{c, \{c^-_k\}_{k=1}^K \sim \rho^{K+1}
        }
        }
        \underbrace{
            \ln (\mathrm{Col} + 1)
        }_{\text{collision}}
        \Big\} + d(\fbf).
        \label{eq:proposed-bound}
    \end{align}
\end{restatable}
The proof is found in~\cref{sec:proof-proposed-lower-bound}.

\cref{theorem:proposed-lower-bound} tells us that probability $\upsilon_{K+1}$ converges to $1$ by increasing the number of negative samples $K$; as a result, the self-supervised loss is more likely to contain the supervised loss and the collision term.
The sub-class loss contributes to the self-supervised loss when $K$ is small as in~\cref{eq:decompose-sup}.
Let us consider the example value of $\upsilon$ with the same setting discussed in~\cref{sec:limitation-curl}.
When $\upsilon_{K=32} \approx 0.719$, i.e., $71.9\%$ training samples contribute the supervised loss.

\subsection{Increasing $K$ does not Spread Normalized Features within Same Latent Class}
\label{sec:collison-upper-bound}
We argue that the feature representations do not have a large within-class variance on the feature space by increasing $K$ in practice.
To do so, we show the upper bound of the collision term in the lower bounds to understand the effect of large negative samples.
\begin{restatable}[Upper Bound of Collision Term]{corollary}{UpperBoundCollision}
    Given a latent class $c$, $K$ negative classes $\{ c_k^- \}_{k=1}^K$, and feature extractor $\fbf$,
    \begin{align}
        \ln \left(
            \mathrm{Col} \left(c, \{c^{-}_{k}\}_{k=1}^K \right)  + 1
        \right)
        \leq
        \alpha + \beta
        \Ebb_{\substack{
            \xbf \sim \Dcal_c \\
            (\abf,\abf^{+}) \sim \Acal^2
        }}
        \Ebb_{
            \substack{
                \xbf' \sim \Dcal_c
                \\
                \abf' \sim \Acal
            }
        }
        \left| \fbf(\abf(\xbf)) \cdot \left[
            \fbf(\abf'(\xbf')) - \fbf(\abf^{+} (\xbf))
        \right] \right|,
        \label{eq:upper-bound-collision}
    \end{align}
    where $\alpha$ and $\beta$ are non-negative constants depending on the number of duplicated latent classes.
    \label{corollary:informal}
\end{restatable}
The proof is found in~\cref{appendix:collision-corollary}.
A similar bound is shown by~\citet[Lemma 4.4]{Arora2019ICML}.

Intuitively, we expect that two feature representations in the same latent class tend to be dissimilar by increasing $K$.
However, \cref{eq:upper-bound-collision} converges even if $K$ is small in practice.
Let us consider the condition when this upper bound achieves the maximum.
Since $\fbf(\abf(\xbf))$ and $\fbf(\abf^{+}(\xbf))$ are computed from the same input sample with different data augmentations, their inner product tends to be $1$;
thus, $\fbf(\abf'(\xbf'))$ is located in the opposite direction of $\fbf(\abf(\xbf))$.
But, it is not possible to learn such representations because of the definition of InfoNCE with normalized representation and the dimensionality of feature space: Equilateral dimension.
We confirm that \cref{eq:upper-bound-collision} without $\alpha$ and $\beta$ does not increase by increasing $K$ (see~\cref{tab:bounds} and \cref{sec:additional_experiments} for more analysis).

\subsection{Small $K$ can Give Consistent Loss Function for Supervised Task}
\label{sec:sub-class-loss-analysis}
As we shown in~\cref{theorem:proposed-lower-bound},
$L_{\mathrm{Info}}$ with large $K$ can be viewed as an upper bound of $L^{\mu}_{\mathrm{sup}}$.
However, $L_{\mathrm{Info}}$ with smaller $K$ can still yield good feature representations for downstream tasks on ImageNet-1K as reported by~\citet{Chen2020ICML}.
\citet{Arora2019ICML} also reported similar results on CIFAR-100 with contrastive losses.%
\footnote{\citet[Table D.1]{Arora2019ICML} mentioned that this phenomenon is not covered by the CURL framework.}
To shed light on this smaller $K$ regime, we focus on the class distributions of both datasets, ImageNet-1K and CIFAR-100, which are almost uniform.
\begin{restatable}[Optimality of $L_\mathrm{sub}$]{proposition}{LSubOptimality}
    \label{proposition:optimal_L_sub}
    Suppose $\Ccal \supseteq \Ycal$ and $\rho$ is uniform: $\forall c \in \Ccal, \rho(c) = 1 / |\Ccal|$.
    Suppose a constant function $\qbf: \xbf \in \Xcal \stackrel{\qbf}{\mapsto} \left[q_1, \ldots, q_{|\Ccal|} \right]^\top$.
    Optimal $\qbf^*$ is a constant function that outputs a vector with the same value if and only if it minimizes $\Ebb_{ c, \{c^-_k\}_{k=1}^K \sim \rho^{K+1} } L_\mathrm{sub}(\qbf^*, \Ccal_{\mathrm{sub}})$.
    The optimal $\qbf^*$ is also the minimizer of $L_{\mathrm{sup}}$.
\end{restatable}
The proof is found in~\cref{appendix:proof-optimality-of-L_sub} inspired by~\citet[Proposition 2]{Titsias2016NeurIPS}.

\cref{proposition:optimal_L_sub} \textit{does not} argue that self-supervised representation learning algorithms fail to learn feature representations on a dataset with a non-uniform class distribution.
This is because we perform a supervised algorithm on a downstream dataset after representation learning in general.

\subsection{Relation to Clustering-based Self-supervised Representation Learning}
During the minimization of $\Lin$, we cannot minimize $L^{\mu}_{\mathrm{sup}}$ in~\cref{eq:proposed-bound} directly since we cannot access supervised and latent classes.
Interestingly, we find a similar formulation in clustering-based self-supervised representation learning algorithms.
\begin{remark}
    Clustering-based self-supervised representation learning algorithms, such as \texttt{\upshape DeepCluster}~\citep{Caron2018ECCV}, \texttt{\upshape SeLa}~\citep{Asano2020ICLR}, \texttt{\upshape SwAV}~\citep{Caron2020NeurIPS}, and \texttt{\upshape PCL}~\citep{Li2021ICLR},
    use prototype representations instead of $\zbf^{+}, \{ \zbf^-_{k}\}_{k=1}^K $ by applying unsupervised clustering on feature representations.
    This procedure is justified as the approximation of latent class's mean representation $\mubf_c$ with a prototype representation to minimize the supervised loss in~\cref{eq:proposed-bound} rather than~\cref{eq:test-self-sup-loss},
    as a result, the mini-batch size does not depend on the number of negative samples.
\end{remark}
This replacement supports the empirical observation in \citet[Section 4.3]{Caron2020NeurIPS},
where the authors reported \texttt{SwAV} maintained top-1 accuracy on ImageNet-1K with a small mini-batch size, $200$, compared to a large mini-batch, $4\,096$.
On the other hand, \texttt{SimCLR}~\citep{Chen2020ICML} did not.

\section{Experiments}
\label{sec:experiments}

We numerically confirm our theoretical findings by using \texttt{SimCLR}~\citep{Chen2020ICML} on the image classification tasks.
\cref{sec:additional_experiments} contains NLP experiments and some analyses that have been omitted due to the lack of space.
We used datasets with a relatively small number of classes to compare bounds.
Our experimental codes are available online.%
\footnote{
    \url{https://github.com/nzw0301/Understanding-Negative-Samples}.
    We used Hydra~\citep{Yadan2019Hydra}, GNU Parallel~\citep{tange_2020_4284075}, Scikit-learn~\citep{JMLR:v12:pedregosa11a}, Pandas~\citep{pandas2020}, Matplotlib~\citep{Hunter2007matplotlib}, and seaborn~\citep{Waskom2021seaborn} in our experiments.
}

\paragraph{Datasets and Data Augmentations}
We used the CIFAR-10 and CIFAR-100~\citep{Krizhevsky2009techrep} image classification datasets with the original $50\,000$ training samples for both self-supervised and supervised training and the original $10\,000$ validation samples for the evaluation of supervised learning.
We used the same data augmentations in the CIFAR-10 experiment by~\citet{Chen2020ICML}:
random resize cropping with the original image size,
horizontal flipping with probability $0.5$,
color jitter with a strength parameter of $0.5$ with probability $0.8$,
and grey scaling with probability $0.2$.

\paragraph{Self-supervised Learning}
We mainly followed the experimental setting provided by~\citet{Chen2020ICML} and its implementation.%
\footnote{\url{https://github.com/google-research/simclr}}
We used ResNet-18~\citep{He2016CVPR} as a feature encoder without the last fully connected layer.
We replaced the first convolution layer with the convolutional layer with $64$ output channels, the stride size of $1$, the kernel size of $3$, and the padding size of $3$.
We removed the first max-pooling from the encoder, and we added a non-linear projection head to the end of the encoder.
The projection head consisted of the fully connected layer with $512$ units, batch-normalization~\citep{Loffe2015ICML}, ReLU activation function, and another fully connected layer with $128$ units and without bias.

We trained the encoder by using PyTorch~\citep{Paszke2019NeurIPS}'s distributed data-parallel training~\citep{Li2018VLDB} on four GPUs, which are NVIDIA Tesla P100 on an internal cluster.
For distributed training, we replaced all batch-normalization with synchronized batch-normalization.
We used stochastic gradient descent with momentum factor of $0.9$ on $500$ epochs.
We used LARC~\citep{You2017techrep}%
\footnote{\url{https://github.com/NVIDIA/apex}},
and its global learning rate was updated by using linear warmup at each step during the first $10$ epochs,
then updated by using cosine annealing without restart~\citep{Loshchilov2017ICLR} at each step until the end.
We initialized the learning rate with $(K+1)/256$.
We applied weight decay of $10^{-4}$ to all weights except for parameters of all synchronized batch-normalization and bias terms.
The temperature parameter was set to $t=0.5$.

\paragraph{Linear Evaluation}
We report the validation accuracy of two linear classifiers: mean classifier and linear classifier.
We constructed a mean classifier by averaging the feature representations of $\zbf$ per supervised class.
We applied one data augmentation to each training sample, where the data augmentation was the same as in that the self-supervised learning step.
For linear classifier $\gbf$,
we optimized $\gbf$ by using stochastic gradient descent with Nesterov's momentum~\citep{Sutskever2013ICML} whose factor is $0.9$ with $100$ epochs.
Similar to self-supervised training, we trained the classifier by using distributed data-parallel training on the four GPUs with $512$ mini-batches on each GPU.
The initial learning rate was $0.3$, which was updated by using cosine annealing without restart~\citep{Loshchilov2017ICLR} until the end.

\paragraph{Bound Evaluation}
We compared the extension bound of CURL~\eqref{eq:decompose-sup} and our bound~\eqref{eq:proposed-bound} by varying the number of negative samples $K$.
We selected $K+1 \in \{32, 64, 128, 256, 512\}$ for CIFAR-10 and $K+1 \in \{128, 256, 384, 512, 640, 786, 896, 1024\}$ for CIFAR-100.
After self-supervised learning, we approximated $\mubf_c$ by averaging $10$ sampled data augmentations per sample on the training dataset and evaluated~\cref{eq:decompose-sup,eq:proposed-bound} with the same negative sample size $K$ as in self-supervised training.%
\footnote{Precisely, $\boldsymbol{\mu}_c = \frac{1}{N_c} \sum_{i=1}^N \Ibb[y_i = c] \frac{1}{10} \sum_{j=1}^{10} \fbf(\abf_j(\xbf_i))$.}
We reported the averaged values over validation samples with $10$ epochs: $(\lfloor 10\,000 / (K+1) \rfloor \times (K+1) \times \mathrm{epoch} )$ pairs of $(\zbf, \Zbf)$.
Note that we used a theoretical value of $\upsilon$ defined by~\cref{eq:coupon-probability} to avoid dividing by zero if a realized $\upsilon$ value is $0$.
We also reported the upper bound of collision~\eqref{eq:upper-bound-collision} without constants $\alpha, \beta$ referred to as ``Collision Bound'' on the training data.
See~\cref{sec:additional_experiments} for details.

\subsection{Experimental Results}
\label{sec:experimental_results}

\begin{table}[t]
    \centering
    \caption{
        The bound values on CIFAR-10/100 experiments with different $K+1$.
        CURL bound and its quantities are shown with $\dagger$.
        The proposed ones are shown without $\dagger$.
        Since the proposed collision values are half of $^\dagger$Collision,
        they are omitted.
        The reported values contain their coefficient except for Collision bound.
    }
    \scriptsize
\begin{tabular}{ll rrrr c rrrr}
    \toprule
    \multicolumn{2}{c}{} & \multicolumn{4}{c}{CIFAR-10} & & \multicolumn{4}{c}{CIFAR-100} \\
    \cmidrule{3-6} \cmidrule{8-11}
    $K+1$                         &                                &    $32$ &   $128$ &   $256$ &   $512$ &     &   $128$ &   $256$ &   $512$ &   $1024$ \\
    \midrule
    $\tau$                         &                                &  $0.96$ &  $1.00$ &  $1.00$ &  $1.00$ &     &  $0.72$ &  $0.92$ &  $0.99$ &   $1.00$ \\
    $\upsilon$                     &                                &  $0.69$ &  $1.00$ &  $1.00$ &  $1.00$ &     &  $0.00$ &  $0.00$ &  $0.62$ &   $1.00$ \\
    $\mu$ acc                      &                                & $72.75$ & $77.22$ & $78.60$ & $80.12$ &     & $32.67$ & $34.25$ & $35.90$ &  $37.44$ \\
    Linear acc                  &                                & $77.13$ & $81.33$ & $82.85$ & $84.13$ &     & $41.95$ & $43.53$ & $45.16$ &  $46.57$ \\
    Linear acc w/o              &                                & $82.02$ & $85.43$ & $86.68$ & $87.66$ &     & $57.92$ & $58.91$ & $59.30$ &  $59.46$ \\
    $\Lin$                       & \cref{eq:test-self-sup-loss}    &  $2.02$ &  $3.29$ &  $3.96$ &  $4.64$ &     &  $3.32$ &  $3.98$ &  $4.66$ &   $5.34$ \\
    $d(\mathbf{f})$                      & \cref{eq:mean-supervised}       & $-1.16$ & $-1.18$ & $-1.18$ & $-1.19$ &     & $-0.99$ & $-0.98$ & $-0.97$ &  $-0.95$ \\
    $^{\dagger}\Lin$ bound          & \cref{eq:decompose-sup}         &  $0.23$ &  $1.41$ &  $2.08$ &  $2.75$ &     &  $0.72$ &  $0.46$ &  $0.78$ &   $1.42$ \\
    \hspace{2ex} $^{\dagger}$Collision &                                &  $1.32$ &  $2.58$ &  $3.26$ &  $3.94$ &     &  $0.69$ &  $1.15$ &  $1.73$ &   $2.37$ \\
    \hspace{2ex} $^{\dagger}L^{\mu}_{\mathrm{sup}}$    &                                &  $0.05$ &  $0.00$ &  $0.00$ &  $0.00$ &     &  $0.00$ &  $0.00$ &  $0.01$ &   $0.00$ \\
    \hspace{2ex} $^{\dagger}L^{\mu}_{\mathrm{sub}}$    &                                &  $0.01$ &  $0.00$ &  $0.00$ &  $0.00$ &     &  $1.03$ &  $0.30$ &  $0.01$ &   $0.00$ \\
    $\Lin$ bound                 & \cref{eq:proposed-bound}        &  $0.39$ &  $1.02$ &  $1.35$ &  $1.69$ &     &  $1.18$ &  $1.53$ &  $1.86$ &   $2.19$ \\
    \hspace{2ex} $L^{\mu}_{\mathrm{sup}}$           &                                &  $0.63$ &  $0.91$ &  $0.90$ &  $0.90$ &     &  $0.00$ &  $0.00$ &  $1.17$ &   $1.94$ \\
    \hspace{2ex} $L^{\mu}_{\mathrm{sub}}$           &                                &  $0.26$ &  $0.00$ &  $0.00$ &  $0.00$ &     &  $1.82$ &  $1.93$ &  $0.79$ &   $0.01$ \\
    Collision bound             & \cref{eq:upper-bound-collision} &  $0.60$ &  $0.61$ &  $0.62$ &  $0.62$ &     &  $0.52$ &  $0.52$ &  $0.51$ &   $0.51$ \\
    \bottomrule
\end{tabular}

    \label{tab:bounds}
\end{table}

\cref{tab:bounds} shows the bound values on CIFAR-10 and CIFAR-100.
We only showed a part of values among different numbers of negative samples due to the page limitation.%
\footnote{\cref{tab:cifar10-all-bound,tab:cifar100-all-bound} in \cref{sec:additional_experiments} provides the comprehensive results.}
We reported mean and linear classifiers' validation accuracy as ``$\mu$ acc'' and ``Linear acc'', respectively.
As a reference, we reported practical linear evaluation's validation accuracy as ``Linear acc w/o'', where we discarded the non-linear projection head from the feature extractor~\citep{Chen2020ICML}.
Since the CURL bound~\eqref{eq:decompose-sup} does not contain $^{\dagger}L^{\mu}_{\mathrm{sup}}$ explicitly, we subtracted $^{\dagger}L^{\mu}_{\mathrm{sup}}$ from $L^{\mu}_{\mathrm{sub}}$ in~\cref{eq:decompose-sup} and reported
$^{\dagger}L^{\mu}_{\mathrm{sup}}$ and subtracted $L^{\mu}_{\mathrm{sub}}$ as $^{\dagger}L^{\mu}_{\mathrm{sub}}$ for the comparison.
We confirmed that CURL bounds converged to $^{\dagger}$Collision with relatively small $K$.
On the other hand, proposed bound values had still a large proportion of supervised loss $L^{\mu}_{\mathrm{sup}}$ with larger $K$.
\cref{fig:comparison_bound_curve} shows \textit{upper} bounds of supervised loss $L_{\mathrm{sup}}$ and the linear accuracy by rearranging~\cref{eq:decompose-sup,eq:proposed-bound}.
The reported values were averaged over three training runs of both self-supervised and supervised steps with different random seeds.
The error bars represented the standard deviation.

\section{Related Work}
\label{sec:related_work}

\subsection{Self-supervised Representation Learning}

Self-supervised learning tries to learn an encoder that extracts generic feature representations from an unlabeled dataset.
Self-supervised learning algorithms solve a pretext task that does not require any supervision and can be easily constructed on the dataset,
such as denoising~\citep{Vincent2008ICML},
colorization~\citep{Zhang2016ECCV,Larsson2016ECCV}
solving jigsaw puzzles~\citep{Noroozi2016ECCV},
inpainting blank pixels~\citep{Pathak2016CVPR},
reconstructing missing channels~\citep{Zhang2017CVPR},
predicting rotation~\citep{Gidaris2018ICLR},
and adversarial generative models~\citep{Donahue2019NeurIPS} for vision;
predicting neighbor words~\citep{Mikolov2013NeurIPS},
generating neighbor sentences~\citep{Kiros2015NeurIPS}, and solving masked language model~\citep{Devlin2019NAACL} for language.
See also recent review articles~\citep{Le-Khac2020IEEEAccess,Schmarje2021IEEEAccess}.

Recent self-supervised learning algorithms for the vision domain mainly solve an instance discrimination task.
\texttt{Exemplar-CNN}~\citep{Dosovitskiy2014NeurIPS} is one of the earlier algorithms that can obtain generic feature representations of images by using convolutional neural networks and data augmentation.
After \citet{Oord2018arXiv} proposed InfoNCE loss function,
many state-of-the-art self-supervised algorithms minimize InfoNCE-based loss function, e.g.,
\texttt{DeepInfoMax}~\citep{Hjelm2019ICLR},  %
\texttt{AMDIM}~\citep{Bachman2019NeurIPS},  %
\texttt{SimCLR}~\citep{Chen2020ICML},
\texttt{CPCv2}~\citep{Henaff2020ICML},  %
\texttt{MoCo}~\citep{He2020CVPR}, %
and \texttt{SwAV}~\citep{Caron2020NeurIPS}.  %

\subsection{Theoretical Perspective}
\label{sec:theretical-related-work}
InfoNCE~\citep{Oord2018arXiv} was initially proposed as a lower bound of intractable mutual information between feature representations by using noise-contrastive estimation (NCE)~\citep{Gutmann2012JMLR}.
Optimizing InfoNCE can be considered as maximizing the InfoMax principle~\citep{Linsker1988Computer}.
The history of InfoMax-based self-supervised representation learning dates back to more than $30$ years ago~\citep{Hinton1990Becker}. %
However, this interpretation does not directly explain the generalization for a downstream task.
Indeed, \citet{Tschannen2020ICLR} empirically showed that the performances of classification on downstream tasks and mutual information estimation are uncorrelated with each other.
\citet{Kolesnikov2019CVPR} also reported a similar relationship between the performances of linear classification on downstream tasks and of pretext tasks.
\citet{McAllester2020AISTATS} theoretically showed the limitation of maximizing the lower bounds of mutual information -- accurate approximation requires an exponential sample size.

As the most related work, \citet{Arora2019ICML} provide the first theoretical analyses to explain the generalization of the CURL.
It is worth noting that our analysis focuses on the different representation learning setting.
Shortly after our publishing a draft of this paper on arXiv, \citet{Ash2021arXiv} also published a paper on the role of negative samples in CURL bound with InfoNCE loss for fully supervised classification using the coupon collector's problem.
Thus we believe that the coupon collector's problem is a key ingredient to analyze the connection between contrastive learning and supervised classification based on the CURL analysis by \citet{Arora2019ICML}.
Their work was developed independently of ours and their analysis is for contrastive unsupervised learning rather than self-supervised learning that is done in this work.
The proposed bound by \citet{Ash2021arXiv} has also similar issue to \citet{Arora2019ICML} as in our analysis; however, their bound holds with smaller $K$ than $C$.
\citet{Bansal2021ICLR} decomposed the generalization error gap of a linear classifier given feature representations.
\citet{Wang2020ICML} decomposed the self-supervised loss into alignment and uniformity and showed properties of both metrics.
\citet{Li2021NeurIPS} provided an interpretation of InfoNCE through a lens of kernel method.
\citet{Mitrovic2021ICLR} provided another theoretical analysis with causality for instance discriminative self-supervised learning.
\citet{Tosh2021ALT} also analyzed self-supervised loss for augmented samples, but they only focused on one negative sample setting.
Recently, \citet{Wei2021ICLR} proposed learning theoretical analysis by introducing ``expansion'' assumption for self-training where (pseudo) labels are generated from the previously learned model.
Learning theory-based analyses were also proposed for other types of self-supervised learning problem such as reconstruction task~\citep{Garg2020NeurIPS,Lee2021NeurIPS} and language modeling~\citep{Saunshi2021ICLR}.
The theoretical analysis on self-supervised representation algorithms without negative samples~\citep{Tian2021ICML} cannot be applied to the contrastive learning setting.

\subsection{Hard Negative Mining}
\label{sec:hard-negative-mining}
In metric learning and contrastive learning, hard negative mining, such as~\citet{Kalantidis2020NeurIPS}, is actively proposed to make training more effective to avoid using inappropriate negative or too easy negative samples.
The current work mainly focuses on the \textit{quality} of negative samples rather than \textit{quantity} of negative samples.
However, removing false-negative samples can reduce the effect of the collision term in our bounds.
Thus our analysis might provide a theoretical justification for hard negative mining.

\section{Conclusion}
\label{sec:conclusion}

We applied the CURL framework to the recent self-supervised representation learning formulation.
We pointed out that the existing framework has difficulty explaining why large negative samples in self-supervised learning improve classification accuracy on a downstream supervised task as in~\cref{fig:comparison_bound_curve}.
We proposed a novel framework using the coupon collector's problem to explain the phenomenon and confirmed our analysis on real-world benchmark datasets.

\paragraph{Limitations}
We did not discuss the properties of data augmentation explicitly in our framework.
Practically, self-supervised representation learning algorithms discard the projection head after self-supervised learning, but our analysis does not cover this procedure.
We believe that extensions to cover these settings are fruitful explorations of future work.

\subsubsection*{Acknowledgments}
This work is supported (in part) by Next Generation AI Research Center, The University of Tokyo.
The experiments were conducted using the SGI Rackable C2112-4GP3/C1102-GP8 (Reedbush-H/L) in the Information Technology Center, The University of Tokyo.
We thank Han Bao, Yoshihiro Nagano, and Ikko Yamane for constructive discussion and Yusuke Tsuzuku for supporting our experiments.
We also thank Junya Honda and Yivan Zhang for \LaTeX~support, specifically, for solving our font issue and MathJax, respectively.
We appreciate anonymous reviewers of ICML 2021 and NeurIPS 2021 for giving constructive suggestions to improve our manuscript.
KN is supported by JSPS KAKENHI Grant Number 18J20470.

{\small
    \bibliography{reference,external}
    \bibliographystyle{abbrvnat}
}

\newpage

\appendix

\section{Notations}
\label{sec:appenxi-notation}

\cref{tab:notations} summarizes our notations.
\begin{table}
    \centering
    \caption{Notations}
    \label{tab:notations}
    \begin{tabular}{ll}
    \toprule
    Symbol    & Description \\
    \midrule
    $c$ & Latent class \\
    $c^-$ & Latent class of a negative sample \\
    $h$ & The dimensionality of feature representation $\zbf$ \\
    $t$ & Temperature parameter in InfoNCE loss~\eqref{eq:NT-Xent} \\
    $y$ & Supervised class \\
    $K$ & The number of negative samples \\
    $N$ & The number of samples used in self-supervised learning \\
    $N_y$ & The number of samples whose label is $y$ \\
    $Y$ & The number of supervised classes \\
    $\alpha$ & Non-negative coefficient in~\cref{eq:upper-bound-collision} \\
    $\beta$ &  Non-negative coefficient in~\cref{eq:upper-bound-collision} \\
    $\tau$ & Collision probability in \cref{proposition:curl-lower-bound} \\
    $\upsilon$ & Probability to draw all latent class defined in~\cref{sec:upsilon-decompose} \\
    $\Ccal$ & Latent classes associated with $\Ycal$ \\
    $\Ccal_{\mathrm{sub}}$ & Function to remove duplicated latent classes \\
    $\Xcal$ & Input space \\
    $\Ycal$ & Supervised class space \\
    $\Ycal_{\mathrm{sub}}$ & Subset of $\Ycal$ \\
    $\Acal$ & Distribution of data augmentations \\
    $\Dcal_y$ & Distribution over $\Xcal$ conditioned on supervised class $y$ \\
    $\Dcal_c$ & Distribution over $\Xcal$ conditioned on latent class $c$ \\
    $\Scal$ & Joint distribution over $\Xcal \times \Ycal$ \\
    $\Scal_{\mathrm{sub}}$ & Joint distribution over $\Xcal \times \Ycal_{\mathrm{sub}}$ \\
    $\rho$ & Probability distribution over $\Ccal$ \\
    $\abf$ & Data augmentation: $\Xcal \rightarrow \Xcal $ \\
    $\abf^{+}$ & Data augmentation to create positive feature representation \\
    $\abf^{-}$ & Data augmentation to create negative feature representation \\
    $\fbf$ & Feature extractor: $\Xcal \rightarrow \Rbb^h$ \\
    $\widehat{\fbf}$ & Trained feature extractor by minimizing $\widehat{L}_{\mathrm{Info}}$ \\
    $\gbf$ & Supervised classifier taken a feature representation $\zbf$: $\Rbb^{h} \rightarrow \Rbb^{Y}$ \\
    $\qbf$ & Function that outputs real-valued vector used in~\cref{proposition:optimal_L_sub} \\
    $\xbf$ & Input sample in $\Xcal$ \\
    $\zbf$ & L2 normalized feature representation: $\fbf(\abf(\xbf))$ \\
    $\zbf^+$ & Positive L2 normalized feature representation: $\fbf(\abf^+(\xbf))$ \\
    $\zbf^-$ & Negative L2 normalized feature representation: $\fbf(\abf^-(\xbf^-))$ \\
    $\Zbf$ & Set of positive and $K$ negative features: $\{\zbf^{+}, \zbf^-_{1}, \ldots, \zbf^-_{K}\}$ \\
    $\mubf(\xbf)$ & Averaged feature representation over $\Acal$: $\Ebb_{\abf \sim \Abf} \fbf(\abf(\xbf))$ \\
    $\mubf_c$ & Mean classifier's weight vector of latent class $c$: $\Ebb_{\xbf \sim \Dcal_c} \mubf(\xbf)$ \\
    $\mubf_y$ & Mean classifier's weight vector of supervised class $y$: $\Ebb_{\xbf \sim \Dcal_y} \mubf(\xbf)$ \\
    $\mathrm{Col} (\cdot, \cdot)$ &  Collision value defined in~\cref{sec:simple-analysis}: $\sum_{k=1}^K \Ibb [c^k = c^-_k ]$ \\
    $d (\cdot)$ & Gap term defined in~\cref{eq:mean-supervised} \\
    $\ell_{\mathrm{Info}}$ & InfoNCE-based self-supervised loss defined by~\cref{eq:NT-Xent} \\
    $\ell_{\mathrm{sub}}^{\mu}$ & Mean classifier's supervised sub-class loss defined by~\cref{eq:single-sub-class-loss} \\
    $L_{\mathrm{Info}}$ & Expected self-supervised loss defined by~\cref{eq:test-self-sup-loss} \\
    $\widehat{L}_{\mathrm{Info}}$ & Empirical self-supervised loss \\
    $L_{\mathrm{sup}}$ & Supervised loss defined by~\cref{eq:supervised-loss} \\
    $L_{\mathrm{sub}}$ & Supervised sub-class loss defined by~\cref{eq:g-sub-class-loss} \\
    $L^{\mu}_{\mathrm{sup}}$ & Supervised loss of mean classifier defined by~\cref{eq:mean-classifier-loss} \\
    $L^{\mu}_{\mathrm{sub}}$ & Sub-class supervised loss of mean classifier by~\cref{eq:sub-class-loss} \\
    $\Ibb[\cdot]$ & Indicator function \\
    $(. \cdot .)$ & Inner product \\
    \bottomrule
    \end{tabular}
\end{table}

\section{Relationship between Latent Classes and Supervised classes}
\label{sec:details-latent-class}

In the main body of the manuscript, we assume that the latent classes $\Ccal$ are a superset of the supervised classes $\Ycal$.
Here, we mention a few scenarios covered by our analysis.

\paragraph{Supervised Class is a Union of Latent Classes}
Suppose that latent classes $\Ccal$ are breeds of dog and cat for image classification and supervised class $y \in \Ycal$ is a union of breeds: ``dog'' or ``cat''.
In this setting, the downstream classifier can mispredict the class label among dog breeds for a dog image.
Therefore the classifier's loss on all latent classes $\Ycal = \Ccal$ is higher than or equal to the same classifier' loss on the union set at worst.

\paragraph{Supervised Class is a Product of Latent Classes}
Another practical scenario is that the supervised class is a product of latent classes.
Suppose that the latent classes are disentangled properties of dog/cat, such as breeds and color.
In this case, a sample can be drawn from different latent classes, for example, a white golden retriever image can be drawn from ``Golden Retriever'' class and ``white'' class.
Suppose that we perform classification over the product of the latent classes such as ``White Golden Retriever'' and ``Black Ragamuffin'' at the supervised step.
Since we use Jensen's inequality to aggregate feature representation into the weights of mean classifier in~\cref{sec:simple-analysis},
our analysis can deal with this case as well.

\section{Relation of Mean Classifier and Linear Classifier}
\label{sec:mean-classifier-and-linear-classifier}

Recall that $\gbf: \Rbb^h \rightarrow \Rbb^{Y}$ that is a function from a feature space to a label space.
Suppose $\widehat{\gbf} = \argmin_{\mathbf{g}} L_{\mathrm{sup}}(\gbf \circ \fbf)$.
The mean classifier's supervised loss is bounded by a loss with $\widehat{\gbf}$:
\begin{align}
    L_{\mathrm{sup}}^{\mu}(\fbf) \geq L_{\mathrm{sup}}(\widehat{\gbf} \circ \fbf).
\end{align}
This relation holds with sub-class loss in~\cref{definition:sub-class-mean-classifier} as well.

Note that the mean classifier is a special case of the linear classifier~\citep[Sec. 2.4]{Snell2017NeurIPS}.
A linear classifier is defined as $\Wbf \fbf(\abf((\xbf))) + \bbf$, where $\Wbf \in \Rbb^{Y \times h}$ and $\bbf \in \Rbb^Y$.
When $\Wbf = [\boldsymbol{\mu}_1, \ldots, \boldsymbol{\mu}_Y]^\top$ and $\bbf = \mathbf{0}$, the linear classifier is equivalent to the mean classifier.

\section{Proofs}
\subsection{Inequality for~\cref{eq:mean-supervised}}
\label{sec:proof-mean-sup}
We show the following inequality used to obtain~\cref{eq:mean-supervised}:
\begin{align}
    &
    \Ebb_{\substack{
        c, \{c_k^-\}_{k=1}^K \sim \rho^{K+1}
        }
    }
    \hspace{1em}
    \Ebb_{\substack{
        \xbf \sim \Dc \\
        \abf \sim \Acal
        }
    }
    \ell_{\mathrm{Info}} \left(
        \zbf,
        \left\{
            \mubf(\xbf),
            \mubf_{c^-_1},
            \ldots,
            \mubf_{c^-_K}
        \right\}
    \right)
    \nonumber
    \\
    \geq
    &
    \Ebb_{
        \substack{
            c, \{c_k^-\}_{k=1}^K \sim \rho^{K+1}
        }
    }
    \hspace{1em}
    \Ebb_{\substack{
        \xbf \sim \Dc \\
        \abf \sim \Acal }
    }
    \ell_{\mathrm{Info}}\left(
        \zbf,
        \left\{
            \mubf_{c},
            \mubf_{c^-_1},
            \ldots,
            \mubf_{c^-_K}
        \right\}
    \right)
    +
    d(\fbf).
    \label{eq:ineq_6}
\end{align}

\begin{proof}
    We replace $\mubf(\xbf)$ with $\mubf_c$ in the left hand side of~\cref{eq:ineq_6}:
    \begin{align}
        &
        \Ebb_{\substack{
            c, \{c_k^-\}_{k=1}^K \sim \rho^{K+1}
            }
        }
        \hspace{1em}
        \Ebb_{\substack{
            \xbf \sim \Dc \\
            \abf \sim \Acal
            }
        }
        \ell_{\mathrm{Info}} \left(
            \zbf,
            \left\{
                \mubf(\xbf),
                \mubf_{c^-_1},
                \ldots,
                \mubf_{c^-_K}
            \right\}
        \right)
        \nonumber
        \\
        =
        &
        \Ebb_{\substack{
            c, \{c_k^-\}_{k=1}^K \sim \rho^{K+1}
            }
        }
        \hspace{1em}
        \Ebb_{\substack{
            \xbf \sim \Dc \\
            \abf \sim \Acal
            }
        }
        \ell_{\mathrm{Info}}\left(
            \zbf,
            \left\{
                \mubf(\xbf),
                \mubf_{c^-_1},
                \ldots,
                \mubf_{c^-_K}
            \right\}
        \right)
        \nonumber
        \\
        &
        +
        \Ebb_{
            \substack{
                c, \{c_k^-\}_{k=1}^K \sim \rho^{K+1} \\
            }
        }
        \hspace{1em}
        \Ebb_{
            \substack{
                \xbf \sim \Dc \\
                \abf \sim \Acal
            }
        }
        \ell_{\mathrm{Info}}\left(
            \zbf,
            \left\{
                \mubf_{c},
                \mubf_{c^-_1},
                \ldots,
                \mubf_{c^-_K}
            \right\}
        \right)
        \nonumber
        \\
        &
        -
        \Ebb_{
            \substack{
                c, \{c_k^-\}_{k=1}^K \sim \rho^{K+1} \\
            }
        }
        \hspace{1em}
        \Ebb_{
            \substack{
                \xbf \sim \Dc \\
                \abf \sim \Acal
            }
        }
        \ell_{\mathrm{Info}}\left(
            \zbf,
            \left\{
                \mubf_{c},
                \mubf_{c^-_1},
                \ldots,
                \mubf_{c^-_K}
            \right\}
        \right)
        \nonumber
        \\
        =&
        \Ebb_{
            \substack{
                c, \{c_k^-\}_{k=1}^K \sim \rho^{K+1} \\
            }
        }
        \hspace{1em}
        \Ebb_{
            \substack{
                \xbf \sim \Dc \\
                \abf \sim \Acal
            }
        }
        \ell_{\mathrm{Info}}\left(
            \zbf,
            \left\{
                \mubf_{c},
                \mubf_{c^-_1},
                \ldots,
                \mubf_{c^-_K}
            \right\}
        \right)
        \nonumber
        \\
        &
        +
        \underbrace{
        \Ebb_{
            \substack{
                c, \{c_k^-\}_{k=1}^K \sim \rho^{K+1} \\
            }
        }
        \hspace{1em}
        \Ebb_{
            \substack{
                \xbf \sim \Dc \\
                \abf \sim \Acal
            }
        }
        \left[
        \ell_{\mathrm{Info}} \left(
            \zbf,
            \left\{
                \mubf(\xbf),
                \mubf_{c^-_1},
                \ldots,
                \mubf_{c^-_K}
            \right\}
        \right)
        -
        \ell_{\mathrm{Info}}\left(
            \zbf,
            \left\{
                \mubf_{c},
                \mubf_{c^-_1},
                \ldots,
                \mubf_{c^-_K}
            \right\}
        \right)
        \right]
        }_{\text{Gap term}}.
        \label{eq:replace_mu}
    \end{align}
    When the gap term is non-negative: $\zbf \cdot \mubf(\xbf) \leq \zbf \cdot \mubf_c$, we drop it from~\cref{eq:replace_mu} to obtain the lower bound~\eqref{eq:ineq_6} with $d(\fbf)=0$.
    Thus we consider a lower bound of the gap term when the gap term is negative, $\zbf \cdot \mubf(\xbf) > \zbf \cdot \mubf_c$, to guarantee~\cref{eq:replace_mu}.
    As a general case, we show the bound with temperature parameter $t$.
    \begin{align}
        & \text{The gap term in~\cref{eq:replace_mu}}
        \nonumber
        \\
        & =
        \Ebb_{\substack{c, \{c^-_k\}_{k=1}^K \sim \rho^{K+1}}}
        \hspace{1em}
        \Ebb_{\substack{\xbf \sim \Dc \\ \abf \sim \Acal }}
        \left[
            \zbf \cdot (\mubf_c - \mubf(\xbf)) / t
        \right]
        +
        \underbrace{
        \ln
        \frac{
            \exp \left(  \zbf \cdot \mubf(\xbf) / t \right) + \sum_{k=1}^K \exp \left( \zbf \cdot \mubf_{c^-_k} / t \right)
        }{
            \exp( \zbf \cdot \mubf_c / t) + \sum_{k=1}^K \exp \left( \zbf \cdot \mubf_{c^-_k} / t \right)
        }
        }_{\text{Non-negative}}
        \nonumber
        \\
        & >
        \hspace{1em}
        \frac{1}{t}
        \Ebb_{\substack{c \sim \rho }}
        \Ebb_{\substack{\xbf \sim \Dc \\ \abf \sim \Acal }}
        \left[
            \zbf \cdot (\mubf_c - \mubf(\xbf))
        \right]
        \nonumber
        \\
        & =
        \hspace{1em}
        \frac{1}{t}
        \Ebb_{\substack{c \sim \rho}}
        \Ebb_{\substack{\xbf \sim \Dc }}
        \left[
            \Ebb_{\abf \sim \Acal }
            \left[ \fbf(\abf(\xbf)) \right]
            \cdot \left(
                \Ebb_{\substack{\xbf^+ \sim \Dc \\ \abf_1^{+} \sim \Acal }}
                \left[ \fbf(\abf_1^+(\xbf^+)) \right]
                -
                \Ebb_{\abf_2^+ \sim \Acal}
                \left[ \fbf(\abf_2^+ (\xbf)) \right]
            \right)
        \right]
        \\
        &
        \coloneqq
        \hspace{1em}
        d(\fbf)
        .
        \nonumber
    \end{align}
\end{proof}

\subsection{Proof of~\cref{proposition:curl-lower-bound}}
\label{sec:proof-proposition}

\CURLLowerBound*

We apply the proof of Theorem B.1 in~\citet{Arora2019ICML} to the self-supervised learning setting.
\begin{proof}
    Before showing the inequality, we define the following sub-class loss with mean classifier for single $\zbf$ with $c$ and $\Ccal_{\mathrm{sub}}(\{c, c^-_1, \ldots, c^-_K\})$:
    \begin{align}
        \ell_{\mathrm{sub}}^{\mu}(\zbf, c, \Ccal_\mathrm{sub})
        \coloneqq - \ln \frac{\exp \left(\zbf \cdot \mubf_{c} \right) }{
            \sum_{j \in \Ccal_\mathrm{sub}}
            \exp \left(\zbf \cdot \mubf_{j} \right)
            }.
        \label{eq:single-sub-class-loss}
    \end{align}
    We start from \cref{eq:mean-supervised}.
    \begin{align}
        \eqref{eq:mean-supervised}
        =
        (1-\tau_K)
        &
        \Ebb_{\substack{c, \{c^-_k \}_{k=1}^K \sim \rho^{K+1}}}
        \hspace{1em}
        \Ebb_{\substack{\xbf \sim \Dc \\ \abf \sim \Acal}}
        \left[
            \ell_{\mathrm{Info}} \left(\zbf,
            \left\{
                \mubf_{c},
                \mubf_{c^-_1},
                \ldots,
                \mubf_{c^-_K}
            \right\} \right) \mid \mathrm{Col} \left(c, \{c^-_k\}_{k=1}^{K} \right) = 0
        \right]
        \nonumber
        \\
        +
        \tau_K
        &
        \Ebb_{\substack{c, \{c^-_k \}_{k=1}^K \sim \rho^{K+1}}}
        \hspace{1em}
        \Ebb_{\substack{\xbf \sim \Dc \\ \abf \sim \Acal}}
            \left[
                \ell_{\mathrm{Info}} \left(\zbf, \left\{
                \mubf_{c},
                \mubf_{c^-_1},
                \ldots,
                \mubf_{c^-_K}
            \right\} \right) \mid \mathrm{Col} \left(c, \{c^-_k\}_{k=1}^{K} \right) \neq 0
            \right]
        + d(\fbf)
        \nonumber
        \\
        \geq
        (1-\tau_K)
        &
        \Ebb_{\substack{c, \{c^-_k \}_{k=1}^K \sim \rho^{K+1}}}
        \hspace{1em}
        \Ebb_{\substack{\xbf \sim \Dc \\ \abf \sim \Acal}}
        [
            \ell^{\mu}_{\mathrm{sub}}(\zbf, c, \Ccal_{\mathrm{sub}} )
            \mid \mathrm{Col} = 0
        ]
        \nonumber
        \\
        + \tau_K
        &
        \Ebb_{\substack{c, \{c^-_k \}_{k=1}^K \sim \rho^{K+1}}}
        \hspace{1em}
        \Ebb_{\substack{\xbf \sim \Dc \\ \abf \sim \Acal}}
        [
            \ell^{\mu}_{\mathrm{sub}}(\zbf, c, \{ c \}^{\mathrm{Col}+1} ) \mid \mathrm{Col} \neq 0
        ]
        + d(\fbf)
        \nonumber
        \\
        =
        (1-\tau_K)
        &
        \Ebb_{\substack{c, \{c^-_k \}_{k=1}^K \sim \rho^{K+1} }}
        \hspace{1em}
        [
            L^{\mu}_{\mathrm{sub}}(\zbf, \Ccal_{\mathrm{sub}} ) \mid \mathrm{Col} = 0
        ]
        \nonumber
        \\
        + \tau_K
        &
        \Ebb_{\substack{c, \{c^-_k \}_{k=1}^K \sim \rho^{K+1} }}
        \hspace{1em}
        [
            \ln (\mathrm{Col} + 1)
            \mid \mathrm{Col} \neq 0
        ]
        + d(\fbf),
    \end{align}
    where $\{ c \}^{\mathrm{Col}+1}$ is a bag of latent classes that contains only $c$ and the number of $c$ is $\mathrm{Col}(c, \{c^-_k\}_{k=1}^{K}) + 1$.

    The first equation is done by using collision probability $\tau_K$ conditioned on $\mathrm{Col}$.
    The inequality is obtained by removing duplicated latent classes from the first term and removing the other latent classes from the second term.
    Note that InfoNCE is a monotonically increasing function; its value decreases by removing any elements in the negative features in the loss.
\end{proof}

\subsection{Proof of~\cref{theorem:proposed-lower-bound}}
\label{sec:proof-proposed-lower-bound}

\ProposedBound*
\begin{proof}
    We start from~\cref{eq:mean-supervised}.
    \begin{align}
        \eqref{eq:mean-supervised} =&
        \hspace{1em}
        \upsilon_{K+1}
        \Ebb_{\substack{c, \{c^-_k \}_{k=1}^K \sim \rho^{K+1}}}
        \Ebb_{\substack{\xbf \sim \Dc \\ \abf \sim \Acal}}
        \left[
            \ell_{\mathrm{Info}}
                \left(
                    \zbf,
                    \left\{
                        \mubf_{c},
                        \mubf_{c^-_1},
                        \ldots,
                        \mubf_{c^-_K}
                    \right\}
                \right)
                \mid \Ccal_{\mathrm{sub}}(\{c, c^-_1, \ldots, c^-_K\}) = \Ccal
        \right]
        \nonumber
        \\
        & +
        (1-\upsilon_{K+1})
        \Ebb_{\substack{c, \{c^-_k \}_{k=1}^K \sim \rho^{K+1}}}
        \Ebb_{\substack{\xbf \sim \Dc \\ \abf \sim \Acal}}
        \left[ \ell_{\mathrm{Info}}
            \left(
                \zbf,
                \left\{
                    \mubf_{c},
                    \mubf_{c^-_1},
                    \ldots,
                    \mubf_{c^-_K}
                \right\}
            \right)
            \mid \Ccal_{\mathrm{sub}}(\{c, c^-_1, \ldots, c^-_K\}) \neq \Ccal
        \right]
        + d(\fbf)
        \label{eq:part-proof-7}
        \\
        \geq
        &
        \hspace{1em}
        \frac{1}{2}
        \Big\{
        \upsilon_{K+1}
        \Ebb_{\substack{c, \{c^-_k \}_{k=1}^K \sim \rho^{K+1}}}
        \hspace{1em}
        \Ebb_{\substack{\xbf \sim \Dc \\ \abf \sim \Acal}}
        \left[
            \ell_{\mathrm{sub}}^{\mu} (\zbf, c, \Ccal_{\mathrm{sub}}) \mid \Ccal_{\mathrm{sub}} = \Ccal
        \right]
        \nonumber
        \\
        & +
        (1-\upsilon_{K+1})
        \Ebb_{\substack{c, \{c^-_k \}_{k=1}^K \sim \rho^{K+1}}}
        \hspace{1em}
        \Ebb_{\substack{\xbf \sim \Dc \\ \abf \sim \Acal}}
        \left[ \ell_{\mathrm{sub}}^{\mu} \left( \zbf, c, \Ccal_{\mathrm{sub}} \right) \mid \Ccal_{\mathrm{sub}} \neq \Ccal \right]
        \nonumber
        \\
        & +
        \Ebb_{\substack{c, \{c^-_k\}_{k=1}^K \sim \rho^{K+1}}}
        \ln \left( \mathrm{Col} \left( c, \{c^-_k\}_{k=1}^{K} \right) + 1 \right)
        \Big\}
        + d(\fbf),
        \label{eq:proof-7}
    \end{align}
    where recall that $\ell_{\mathrm{sub}}^{\mu}$ is defined by \cref{eq:single-sub-class-loss}.

    The first equation is obtained by using the definition of $\upsilon$ to distinguish whether or not the sampled latent classes contain all latent classes $\Ccal$.
    To derive the inequality, we use the following properties of the cross-entropy loss.
    Fixed $\fbf$ and $c$, for all $\Kcal \subseteq \{ 1, \ldots, K \}$,
    the value of $\ell_{\mathrm{Info}}$ holds the following inequality:
    \begin{align}
        \ell_{\mathrm{Info}} \left( \zbf,
        \left\{
            \{ \mubf_{c} \} \cup
            \{
                \mubf_{c^-_k}
            \}_{k=1}^K
        \right\} \right)
        \geq
        \ell_{\mathrm{Info}}
            \left( \zbf,
                \{ \mubf_{c} \} \cup
                \{
                \mubf_{c^-_k}
                \}_{k \in \Kcal}
            \right).
    \end{align}
    Thus we apply the following inequalities to the first and second terms in~\cref{eq:part-proof-7}.
    \begin{align}
        \ell_{\mathrm{Info}} \left( \zbf, \left\{
            \mubf_{c},
            \mubf_{c^-_1},
            \ldots,
            \mubf_{c^-_K}
        \right\} \right)
        &
        \geq
        \frac12
        \left[
            \ell_{\mathrm{sub}}^{\mu} \left( \zbf, c, \Ccal_{\mathrm{sub}} \right) + \ell^{\mu}_{\mathrm{sub}} \left( \zbf, c, \{c\}^{\mathrm{Col}+1}  \right)
        \right]
        \nonumber
        \\
        &
        =
        \frac12
        \left[
            \ell_{\mathrm{sub}}^{\mu} \left( \zbf, c, \Ccal_{\mathrm{sub}} \right) + \ln(\mathrm{Col} + 1)
        \right]
    \end{align}

    By following the notations of $L_{\mathrm{sup}}^{\mu}$ and $L_{\mathrm{sub}}^{\mu}$, we obtain the lower bound~\eqref{eq:proposed-bound} in~\cref{theorem:proposed-lower-bound} from~\cref{eq:proof-7}.
\end{proof}

\subsection{Proof of~\cref{corollary:informal}}
\label{appendix:collision-corollary}
\UpperBoundCollision*

We prove~\cref{corollary:informal} that shows the upper bound of the collision term to understand the effect of the number of negative samples.
Our proof is inspired by~\citet[Lemma A.1]{Arora2019ICML}.
As a general case, we consider the loss function with temperature parameter $t$.

\begin{proof}
    We decompose the self-supervised loss $\Lin$ by using $\upsilon$ to extract the collision term.
    \begin{align}
        & \Lin(\fbf)
        \nonumber
        \\
        &=
        \upsilon_{K+1}
        \left[
            \Ebb_{
                \substack{
                    c, \{c_k^-\}_{k=1}^K \sim \rho^{K+1}
                }
            }
            \Ebb_{
                \substack{
                    \xbf \sim \Dc
                    \\
                    (\abf, \abf^{+}) \sim \Acal^{2}
                }
            }
            \Ebb_{
                \substack{
                    \{ \xbf^-_k \sim \Dcal_{c^-_k} \}_{k=1}^K
                    \\
                    \{ \abf^-_k \}_{k=1}^K \sim \Acal^K
                }
            }
            \left[
                \ell_{\mathrm{Info}}(\zbf, \Zbf) \mid \Ccal_{\mathrm{sub}}(\{c, c^-_1, \ldots, c^-_K\}) = \Ccal
            \right]
        \right]
        \nonumber
        \\
        &
        \hspace{1em}
        +
        (1-\upsilon_{K+1})
        \left[
            \Ebb_{
                \substack{
                    c, \{c_k^-\}_{k=1}^K \sim \rho^{K+1}
                }
            }
            \Ebb_{
                \substack{
                    \xbf \sim \Dc
                    \\
                    (\abf, \abf^{+}) \sim \Acal^{2}
                }
            }
            \Ebb_{
                \substack{
                    \{ \xbf^-_k \sim \Dcal_{c^-_k} \}_{k=1}^K
                    \\
                    \{ \abf^-_k \}_{k=1}^K \sim \Acal^K
                }
            }
            \left[
                \ell_{\mathrm{Info}}(\zbf, \Zbf) \mid \Ccal_{\mathrm{sub}}(\{c, c^-_1, \ldots, c^-_K\}) \neq \Ccal
            \right]
        \right]
        \nonumber
        \\
        & \leq
        \upsilon_{K+1}
        \left[
            \Ebb_{
                \substack{
                    c, \{c_k^-\}_{k=1}^K \sim \rho^{K+1}
                }
            }
            \Ebb_{
                \substack{
                    \xbf \sim \Dc
                    \\
                    (\abf, \abf^{+}) \sim \Acal^{2}
                }
            }
            \Ebb_{
                \substack{
                    \{ \xbf^-_k \sim \Dcal_{c^-_k} \}_{k=1}^K
                    \\
                    \{ \abf^-_k \}_{k=1}^K \sim \Acal^K
                }
            }
            \left[
                \ell_{\mathrm{Info}} \left( \zbf, \{\zbf_{k} \}_{c \neq c^-_k}  \right)
                +
                \ell_{\mathrm{Info}} \left( \zbf, \{\zbf_{k} \}_{c = c^-_k} \right)
                \mid \Ccal_{\mathrm{sub}} = \Ccal
            \right]
        \right]
        \nonumber
        \\
        &
        \hspace{1em}
        + (1-\upsilon_{K+1})
        \left[
            \Ebb_{
                \substack{
                    c, \{c_k^-\}_{k=1}^K \sim \rho^{K+1}
                }
            }
            \Ebb_{
                \substack{
                    \xbf \sim \Dc
                    \\
                    (\abf, \abf^{+}) \sim \Acal^{2}
                }
            }
            \Ebb_{
                \substack{
                    \{ \xbf^-_k \sim \Dcal_{c^-_k} \}_{k=1}^K
                    \\
                    \{ \abf^-_k \}_{k=1}^K \sim \Acal^K
                }
            }
        \left[
            \ell_{\mathrm{Info}} \left( \zbf, \{\zbf_{k} \}_{c \neq c^-_k} \right)
            + \ell_{\mathrm{Info}} \left( \zbf, \{\zbf_{k} \}_{c = c^-_k} \right)
            \mid \Ccal_{\mathrm{sub}} \neq \Ccal
            \right]
        \right]
        \nonumber
        \\
        &
        =
        \hspace{1em}
        \Ebb_{
            \substack{
                c, \{c_k^-\}_{k=1}^K \sim \rho^{K+1}
            }
        }
        \Ebb_{
            \substack{
                \xbf \sim \Dc
                \\
                (\abf, \abf^{+}) \sim \Acal^{2}
            }
        }
        \Ebb_{
            \substack{
                \{ \xbf^-_k \sim \Dcal_{c^-_k} \}_{k=1}^K
                \\
                \{ \abf^-_k \}_{k=1}^K \sim \Acal^K
            }
        }
        \ell_{\mathrm{Info}} \left(\zbf, \left\{ \zbf_{k} \right\}_{c = c^-_k} \right) +
        \mathrm{Reminder}
        \nonumber
        \\
        &
        =
        \hspace{1em}
        \Ebb_{
            \substack{
                c \sim \rho
            }
        }
        \Ebb_{
            \mathrm{Col} \sim \Bcal_c
        }
        \Ebb_{
            \substack{
                \xbf, \{ \xbf_k^- \}_{k=1}^{\mathrm{Col}} \sim \Dcal_c^{\mathrm{Col}+1}
                \\
                (\abf, \abf^{+}, \{ \abf_k^- \}_{k=1}^{\mathrm{Col}}) \sim \Acal^{\mathrm{Col}+2}
            }
        }
        \ell_{\mathrm{Info}}(\zbf, \Zbf) +
        \mathrm{Reminder},
        \label{eq:upper-infoNCE}
    \end{align}
    where $\Bcal_c$ is the probability distribution over $\mathrm{Col}$ conditioned on $c$ with $K$.
    The first equality is done by decomposition with $\upsilon$.
    The inequality is obtained by using the following property of the cross-entropy loss.
    Given $\Kcal \subseteq \{ 1, \ldots, K \}$,
    \begin{align}
        \ell_{\mathrm{Info}}(\zbf, \Zbf)
        \leq
        \ell_{\mathrm{Info}} \left(
            \zbf, \{ \zbf^+ \} \cup \{\zbf_{k}^- \}_{k \in \Kcal}
        \right)
        +
        \ell_{\mathrm{Info}}
        \left(
            \zbf, \{ \zbf^+ \} \cup \{\zbf_{k}^- \}_{k \in \{1, \ldots, K\} \setminus \Kcal}
        \right).
    \end{align}

    We focus on the first term in~\cref{eq:upper-infoNCE}, where the loss takes samples are drawn from the same latent class $c$.
    Fixed $\fbf$ and $c$, let $m = \max_{\zbf_k \in \Zbf } \zbf \cdot \zbf_k / t$ and $\zbf^* = \argmax_{\zbf_k \in \Zbf} \zbf \cdot \zbf_k$.
    \begin{align}
        &\Ebb_{
            \substack{
                c \sim \rho
                \\
                \mathrm{Col} \sim \Bcal_c
            }
        }
        \Ebb_{
            \substack{
                \xbf, \{ \xbf^-_k \}_{k=1}^{\mathrm{Col}} \sim \Dcal_c^{\mathrm{Col}+1}
                \\
                (\abf, \abf^{+}, \{ \abf^-_k \}_{k=1}^{\mathrm{Col}}) \sim \Acal^{\mathrm{Col}+2}
            }
        }
        \ell_{\mathrm{Info}}(\zbf, \Zbf)
        \nonumber
        \\
        &\leq
        \Ebb_{
            \substack{
                c \sim \rho
                \\
                \mathrm{Col} \sim \Bcal_c
            }
        }
        \Ebb_{
            \substack{
                \xbf, \{ \xbf^-_k \}_{k=1}^{\mathrm{Col}} \sim \Dcal_c^{\mathrm{Col}+1}
                \\
                (\abf, \abf^{+}, \{ \abf^-_k \}_{k=1}^{\mathrm{Col}}) \sim \Acal^{\mathrm{Col}+2}
            }
        }
        - \zbf \cdot \zbf^{+} / t + \ln \left[ \exp( \zbf \cdot \zbf^{+} / t) + \mathrm{Col} \exp( m ) \right]
        \nonumber
        \\
        &\leq
        \Ebb_{
            \substack{
                c \sim \rho
                \\
                \mathrm{Col} \sim \Bcal_c
            }
        }
        \Ebb_{
            \substack{
                \xbf, \{ \xbf^-_k \}_{k=1}^{\mathrm{Col}} \sim \Dcal_c^{\mathrm{Col}+1}
                \\
                (\abf, \abf^{+}, \{ \abf^-_k \}_{k=1}^{\mathrm{Col}}) \sim \Acal^{\mathrm{Col}+2}
            }
        }
            \max \left[
                \ln (\mathrm{Col} + 1),
                \ln (\mathrm{Col} + 1) - \zbf \cdot \zbf^{+} / t + m
            \right]
        \nonumber
        \\
        &\leq
        \Ebb_{
            \substack{
                c \sim \rho
                \\
                \mathrm{Col} \sim \Bcal_c
            }
        }
        \Ebb_{
            \substack{
                \xbf, \{ \xbf^-_k \}_{k=1}^{\mathrm{Col}} \sim \Dcal_c^{\mathrm{Col}+1}
                \\
                (\abf, \abf^{+}, \{ \abf^-_k \}_{k=1}^{\mathrm{Col}}) \sim \Acal^{\mathrm{Col}+2}
            }
        }
        \ln (\mathrm{Col} + 1)
        +
        \frac1t | \zbf \cdot \zbf^* - \zbf \cdot \zbf^{+} |
        .
    \end{align}
    Note that $\ln (\mathrm{Col} + 1)$ is the constant depending on the number of duplicated latent classes,
    then we focus on $ | \zbf \cdot \zbf^* - \zbf \cdot \zbf^{+} | $.
    \begin{align}
        \Ebb_{
            \substack{
                \xbf, \{ \xbf^-_k \}_{k=1}^{\mathrm{Col}} \sim \Dcal_c^{\mathrm{Col}+1}
                \\
                (\abf, \abf^{+}, \{ \abf^-_k \}_{k=1}^{\mathrm{Col}}) \sim \Acal^{\mathrm{Col}+2}
            }
        }
        | \zbf \cdot \zbf^* - \zbf \cdot \zbf^{+} |
        \nonumber
        &\leq
        (\mathrm{Col} + 1)
        \Ebb_{\substack{
            \xbf \sim \Dcal_c \\
            (\abf, \abf^{+}) \sim \Acal^2
        }}
        \Ebb_{
            \substack{
                \xbf' \sim \Dcal_c
                \\
                \abf' \sim \Acal
            }
        }
            | \zbf \cdot \zbf' - \zbf \cdot \zbf^{+} |.
    \end{align}
    Therefore, $\alpha = \ln(\mathrm{Col}+1)$ and $\beta = \frac{\mathrm{Col}+1}{t}$.
\end{proof}

\subsection{Proof of~\cref{proposition:optimal_L_sub}}
\label{appendix:proof-optimality-of-L_sub}

\LSubOptimality*

\begin{proof}
    Suppose an observed set of data for supervised sub-class loss $ \{ (\xbf_i, \Ccal_{\mathrm{sub}, i}(\{c_i, c^-_1, \ldots, c^-_K\})) \}_{i=1}^M$, where $\Ccal_{\mathrm{sub}, i}$ is a subset of $\Ccal$ such that $\Ccal_{\mathrm{sub}, i}$ contains $c_i$ of $\xbf_i$ and it holds $2 \leq |\Ccal_{\mathrm{sub},i}| \leq K+1$.
    We take derivatives of the empirical sub-class loss $\frac{1}{M} \sum_{i=1}^M - \ln \frac{\exp(q_{c_i})}{\sum_{j \in \Ccal_{\mathrm{sub}, i} } \exp(q_j) }$ with respect to each element of $\qbf$, then the stationary points are $\forall c \in \Ccal$,
    \begin{align}
        & \sum_{n_1 \in \Ccal \setminus c }
        N_{c, n_1}
        \left(
        1
        -
        2
        \frac{
            \exp(q_c)
        }{
            {\displaystyle
                \sum_{ j \in \{c, n_1\} } \exp(q_j)
            }
        }
        \right)
        \nonumber
        \\
        +
        & \sum_{n_1, n_2 \in \Ccal \setminus c }
        N_{c, n_1, n_2}
        \left(
        1
        -
        3
        \frac{
            \exp(q_c)
        }{
            {\displaystyle
                \sum_{ j \in \{c, n_1, n_2\} } \exp(q_j)
            }
        }
        \right)
        \nonumber
        \\
        \vdots
        \nonumber
        \\
        +
        & \sum_{n_1, \ldots, n_K \in \Ccal \setminus c }
        N_{c, n_1, \ldots, n_K}
        \left(
        1
        -
        (K+1)
        \frac{
            \exp(q_c)
        }{
            {\displaystyle
                \sum_{ j \in \{c, n_1, \ldots, n_K \} } \exp(q_j)
            }
        }
        \right) = 0,
        \label{eq:derivative-subclass-loss}
    \end{align}
    where $N_{c, n_1, \ldots, n_K }$ is the frequency of $\Ccal_{\mathrm{sub}, i}$ such that $\sum_{i=1}^M \Ibb [ \Ccal_{\mathrm{sub}, i} = \{ c_i \} \cup \{n_1, \ldots, n_K\}]$.
    As a result, the optimal $\qbf^*$ is a constant function with the same value.
    For supervised loss defined in~\cref{eq:supervised-loss}, the optimal score function of class $y$ is $q_y = \ln N_y + \mathrm{Constant}$, where $N_y$ is the number of samples whose label is $y$ and $\mathrm{Constant} \in \Rbb$.
    From the uniform assumption, the optimal $\qbf^*$ is the minimizer of $L_{\mathrm{sup}}$.
\end{proof}

\section{Expected Number of Negative Samples to Draw all Supervised Labels}
\label{sec:expeted_number}
We assume that we sample a latent class from $\rho$ independently.
Let $\rho(c) \in [0, 1]$ be a probability that $c$ is sampled.
\citet[Theorem 4.1]{Flajolet1992DAM} show the expected value of $K+1$ to sample all latent class in $\mathcal{C}$ is defined by
\begin{align}
    \Ebb [K+1] = \int_{0}^{\infty}
    \left(  1 - \prod_{c \in \mathcal{C}} \left[ 1-\exp(-\rho(c) x ) \right]  \right)
    \mathrm{d}x.
    \label{eq:coupon-expectation}
\end{align}

\cref{tab:coupon_numbers} shows the expected number of sampled latent classes on a popular classification dataset when the latent class is the same as the supervised class.
For ImageNet, we use the relative frequency of supervised classes in the training dataset as $\rho$.
According to~\cref{tab:coupon_numbers}, the empirical number of negative samples is supposed to be natural, for example, $K \geq 8\,096$ in experiments by~\citet{Chen2020ICML,He2020CVPR}.

\begin{table}
    \centering
    \caption{The expected number of samples to draw all supervised classes.}
    \label{tab:coupon_numbers}
    \begin{tabular}{rrr}
        \toprule
        Dataset & \# classes & $\mathbb{E} [K+1]$ \\
        \midrule
        AGNews  &       $4$ &      $9$ \\
        CIFAR-10  &     $10$ &     $30$ \\
        CIFAR-100 &    $100$ &    $519$ \\
        ImageNet  & $1\,000$ & $7\,709$ \\
        \bottomrule
    \end{tabular}
\end{table}

\section{Additional Experimental Results}
\label{sec:additional_experiments}

\subsection{Experimental Results related to Upper Bound of Collision Term in~\cref{corollary:informal}}

As shown in~\cref{tab:bounds}, the upper bound values of the collision term did not increase by increasing $K$.
We further analyzed representations in terms of the number of negative samples $K$.

As we discuss in the~\cref{sec:collison-upper-bound}, we expect that cosine similar values between samples in the same latent class do not change by increasing $K$.
To confirm that, \cref{fig:cosine_hist_per_class_cifar10,fig:cosine_hist_per_class_cifar100} show histgrams of cosine similarity values between all pairs of feature representations in the same supervised class.
The cosine similarity values do not change by increasing $K$ in practice.
We used feature representation $\widehat{\fbf}(\xbf)$ on the training dataset.
We used the number of bins of each histogram as the square root of the number of cosine similarity values in each class.
We did not use duplicated similarity values: similarity between $\widehat{\fbf}(\xbf_i)$ and $\widehat{\fbf}(\xbf_j)$, where $i \geq j$.
For the results of CIFAR-100 dataset on~\cref{fig:cosine_hist_per_class_cifar100}, we show the histograms for the first $10$ supervised classes due to the page width.

To understand more details of the feature representations, \cref{fig:norm_hist} shows L2 norms of \textit{unnormalized} feature representations extracted from the training data of CIFAR-10 and CIFAR-100.
We used the same feature representation $\widehat{\fbf}(\xbf)$ without L2 normalization as in \cref{fig:cosine_hist_per_class_cifar10,fig:cosine_hist_per_class_cifar100}.
Unlike cosine similarity and the collision upper bound values, the norm values increase by increasing $K$ on CIFAR-10 with all $K$ and CIFAR-100 with smaller $K$.
We used the square root of the number of training samples, $223$, as the number of bins of each histogram.

To show the tendency of cosine similarity and norms by increasing $K$, \cref{fig:relative-change} shows the relative change of a distance between histograms.
\cref{fig:wasserstein_cifar10} shows the relative change values for the CIFAR-10 dataset based on~\cref{fig:cosine_hist_per_class_cifar10,fig:norm_hist_cifar10}.
Similarity, \cref{fig:wasserstein_cifar100} shows the relative change values for the CIFAR-100 dataset based on~\cref{fig:cosine_hist_per_class_cifar100,fig:norm_hist_cifar100}.
For the representations extracted from the CIFAR-10 dataset, we calculated the first Wasserstein distance between the histogram of $K+1=32$ and the other $K$ values with Scipy~\citep{2020SciPy-NMeth}.
We used the distance between the histograms of $K+1=32$ and $K+1=64$ as the reference value of the relative change.
We calculated the averaged Wasserstein distance values among the supervised classes by random seed and took the averaged values over the three different random seeds.
For CIFAR-100, we used the Wasserstein distance between histograms of $K+1=128$ and $K+1=256$ as the reference value.
Note that we used minimum and maximum norm values to unify the range of histograms among different $K$ since L2 norm values are not bounded above.

\begin{figure}[t]
    \centering
    \includegraphics[width=1.0\textwidth]{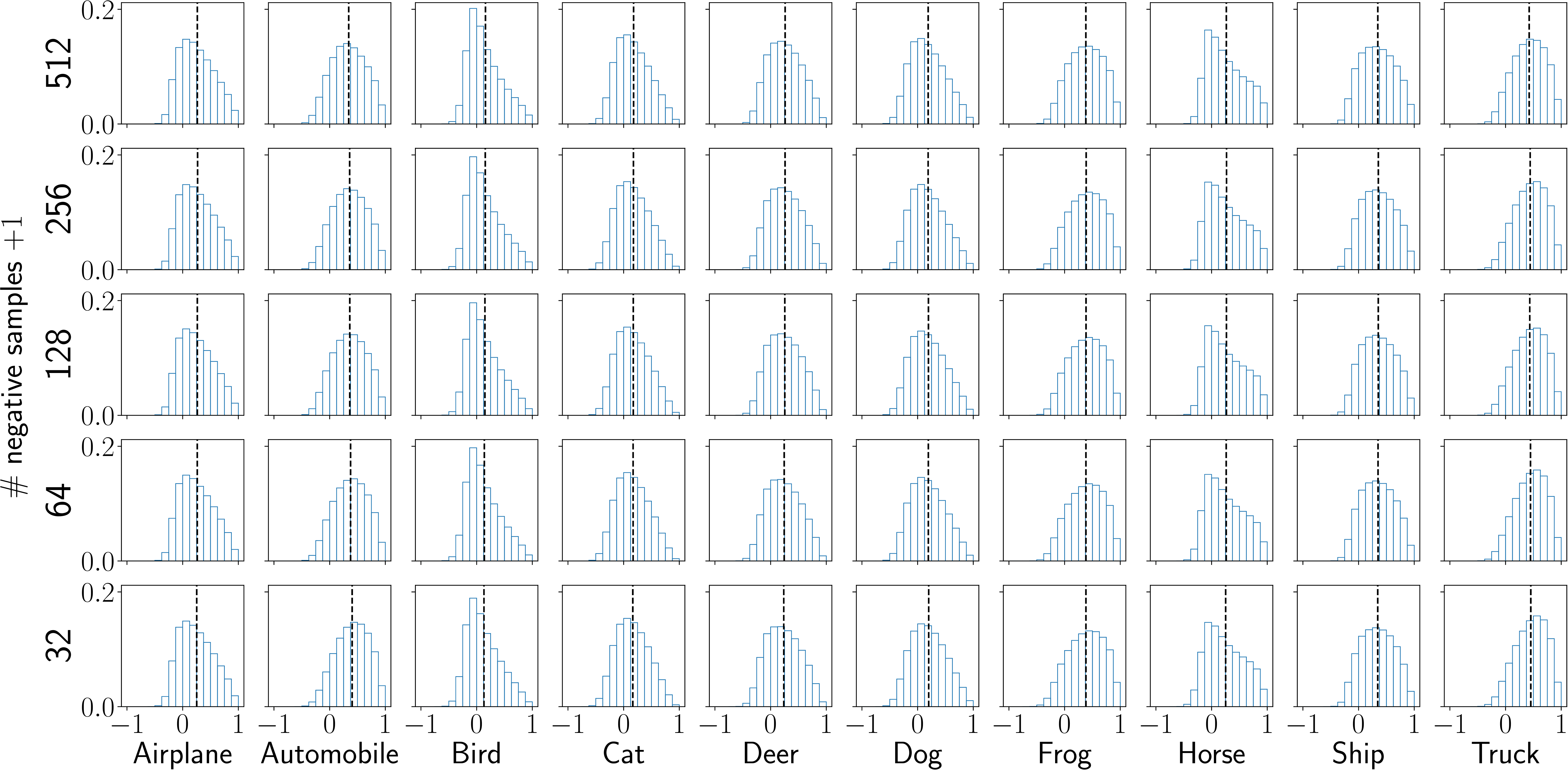}
    \caption{
        Histgrams of cosine similarity between the features in the same supervised class on CIFAR-10.
        The vertical lines represent the mean values.
    }
    \label{fig:cosine_hist_per_class_cifar10}
    \vspace{2ex}
    \includegraphics[width=1.0\textwidth]{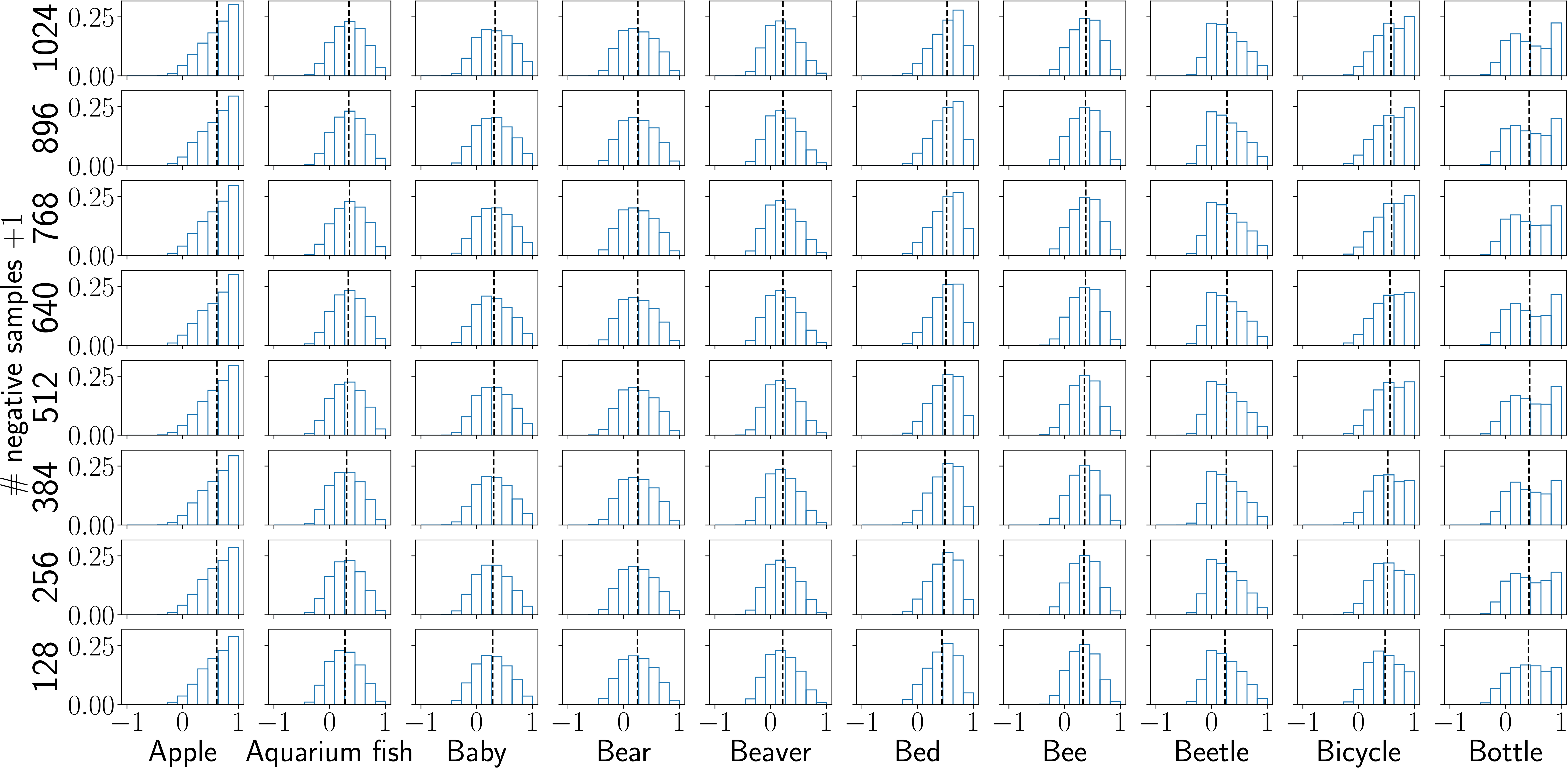}
    \caption{
        Histgrams of cosine similarity between the features in the same supervised class on CIFAR-100.
        The vertical lines represent the mean values.
    }
    \label{fig:cosine_hist_per_class_cifar100}
\end{figure}

\begin{figure}
    \centering
    \begin{subfigure}[t]{0.49\textwidth}
        \centering
        \includegraphics[width=\textwidth]{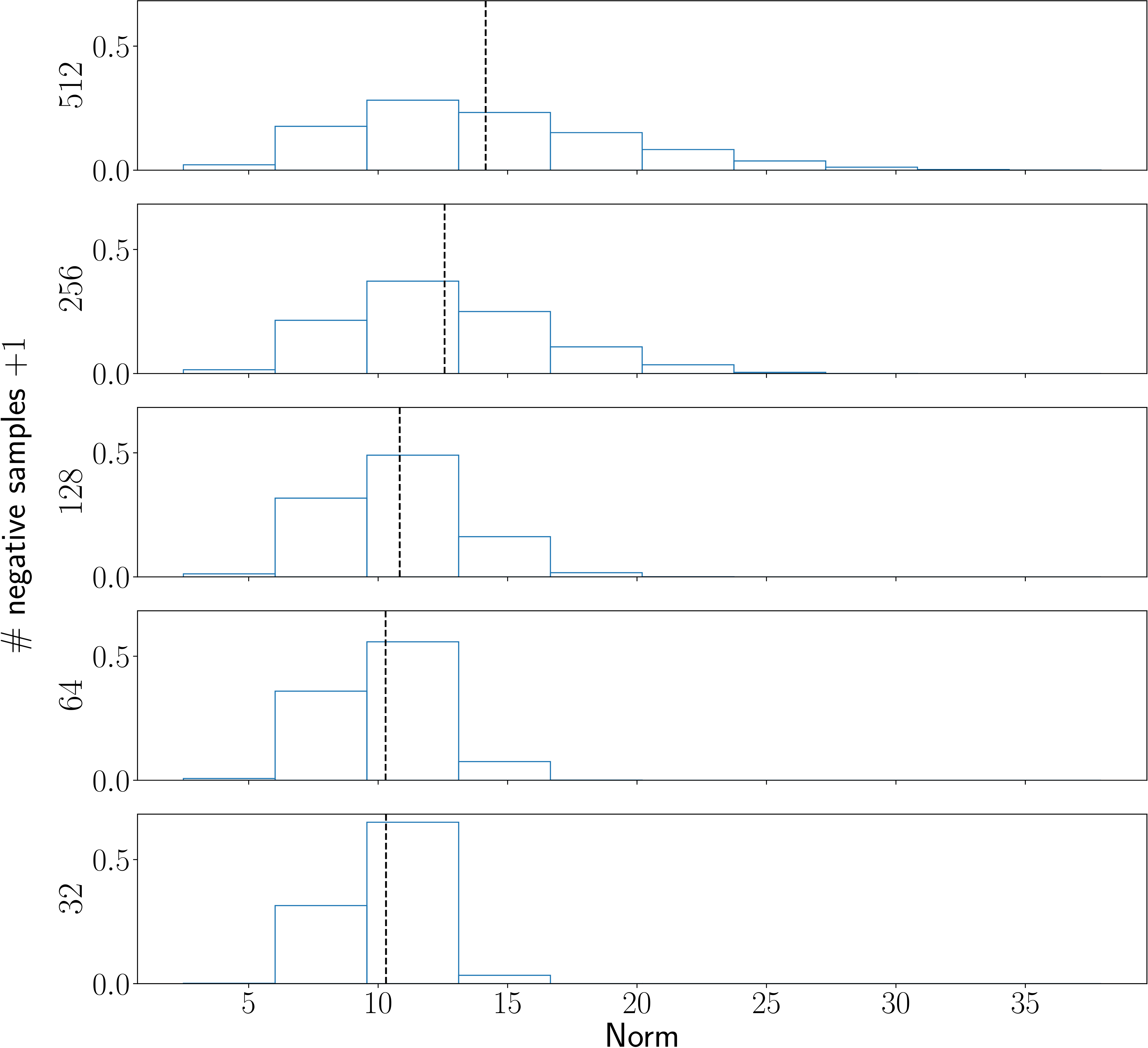}
        \caption{
            CIFAR-10
        }
        \label{fig:norm_hist_cifar10}
    \end{subfigure}
    \begin{subfigure}[t]{0.49\textwidth}
        \centering
        \includegraphics[width=\textwidth]{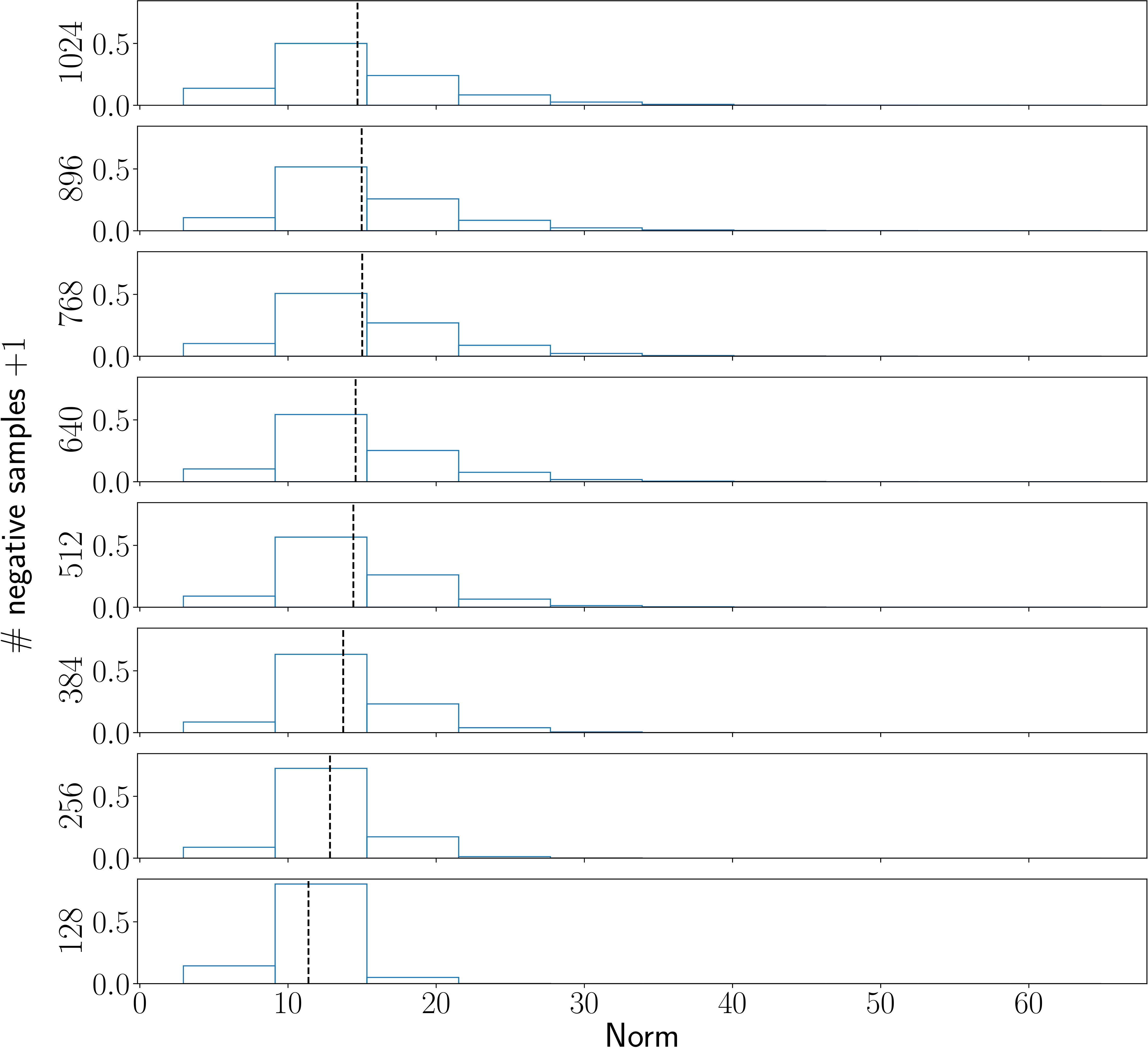}
        \caption{
            CIFAR-100
        }
        \label{fig:norm_hist_cifar100}
    \end{subfigure}
    \caption{
        Histgrams of L2 norm values of unnormalized feature representations.
        The vertical lines represent the mean values.
    }
    \label{fig:norm_hist}
\end{figure}

\begin{figure}
    \centering
    \begin{subfigure}[t]{0.49\textwidth}
        \centering
        \includegraphics[width=\textwidth]{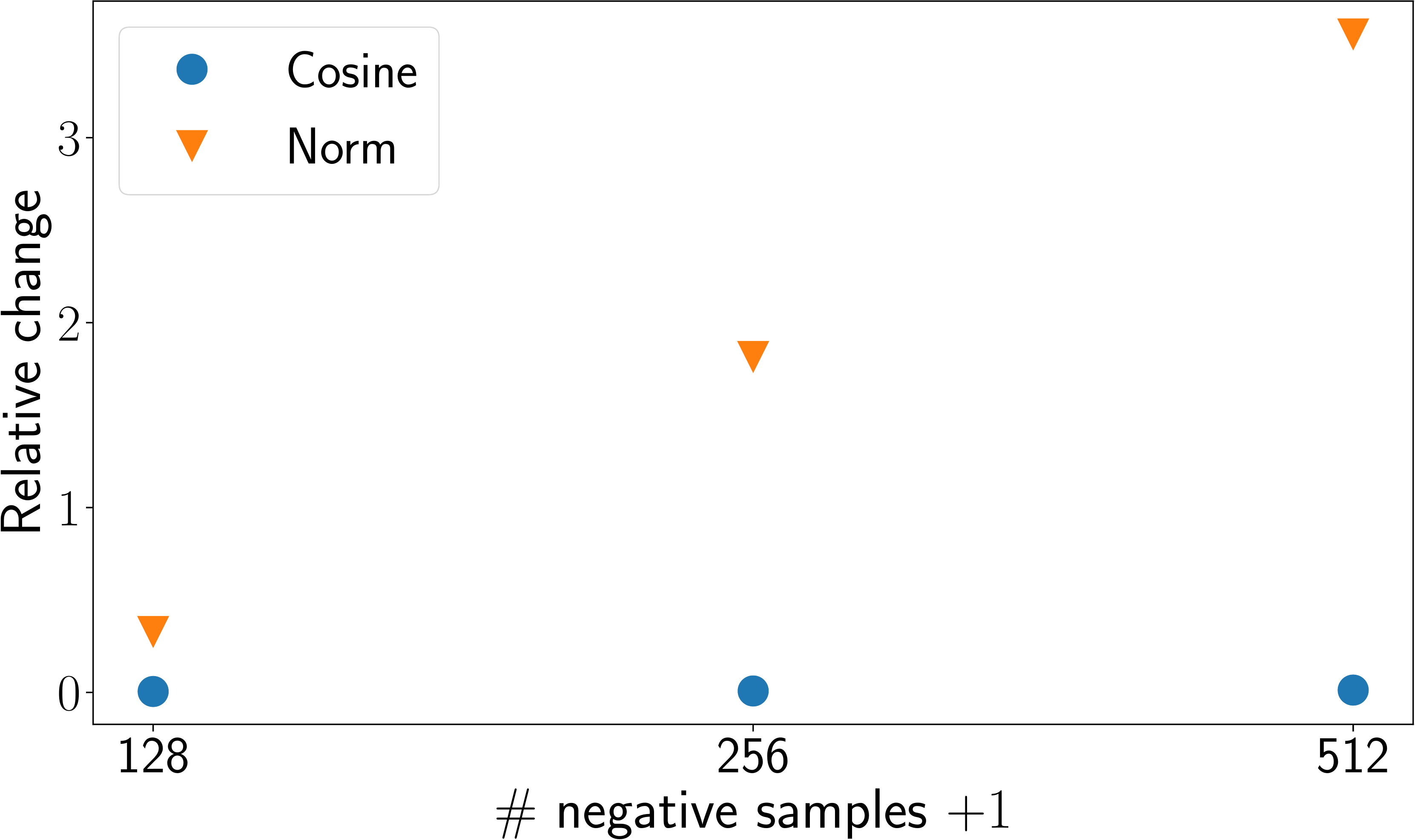}
        \caption{
            Relative change of the first Wasserstein distance between the histogram of $K+1=32$ and the histogram of the other $K+1$ on CIFAR-10.
            The reference value is the distance between the histograms of $K+1=32$ and $K+1=64$.
        }
        \label{fig:wasserstein_cifar10}
    \end{subfigure}
    \hfill
    \begin{subfigure}[t]{0.49\textwidth}
        \centering
        \includegraphics[width=\textwidth]{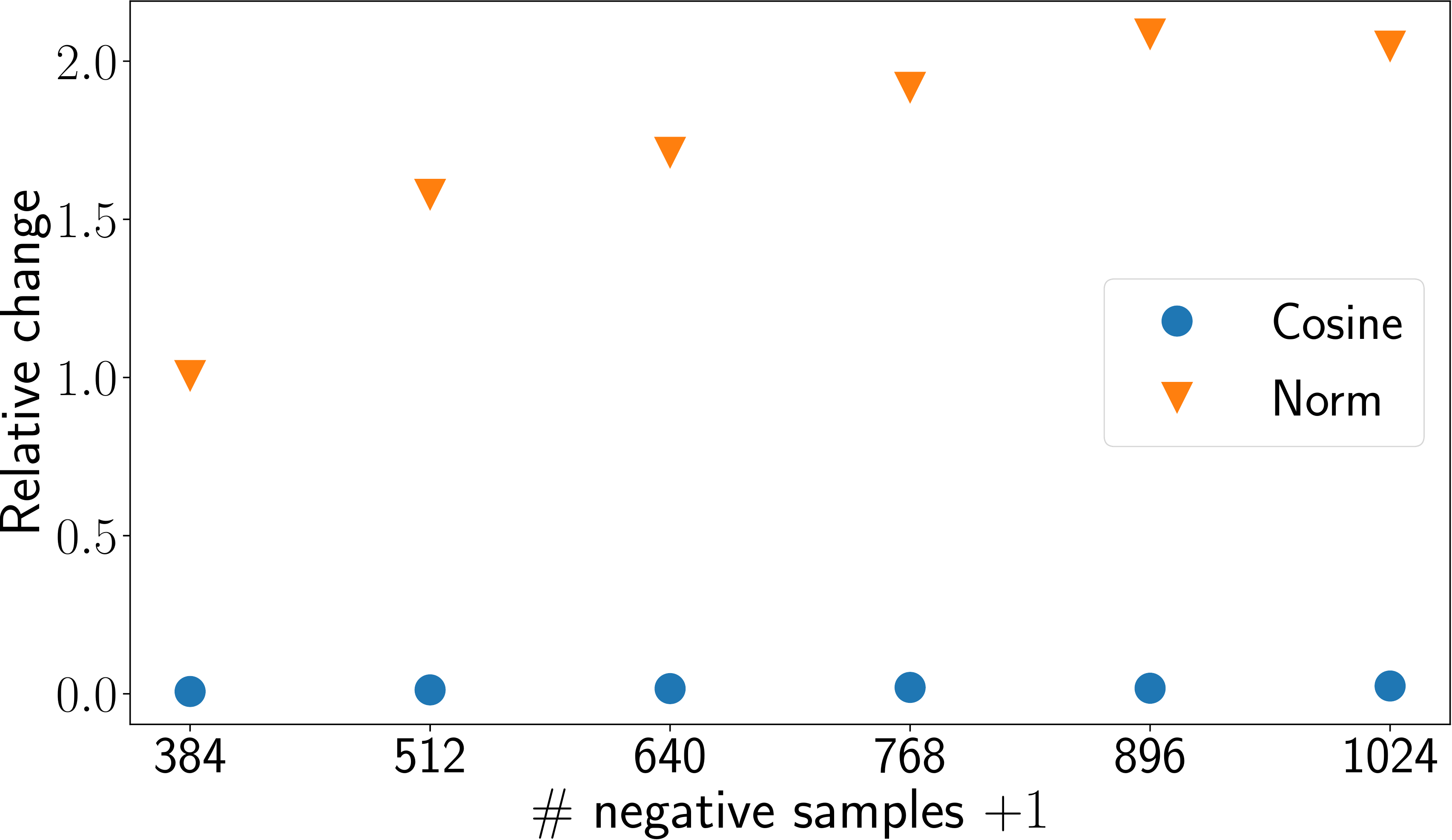}
        \caption{
            Relative change of the first Wasserstein distance between the histogram of $K+1=128$ and the histogram of the other $K+1$ on CIFAR-100.
            The reference is the distance between the histograms of $K+1=128$ and $K+1=256$.
        }
        \label{fig:wasserstein_cifar100}
    \end{subfigure}
    \caption{Relative change of the first Wasserstein distance between histograms by increasing the number of negative samples.}
    \label{fig:relative-change}
\end{figure}

\subsection{Details of Collision bound Calculation on \cref{tab:bounds}}

To calculate the upper bound~\eqref{eq:upper-bound-collision} of the collision term without coefficient terms $\alpha,\beta$,
we sampled two data augmentations $\abf$ and $\abf^+$ per each training sample $\xbf$.
The data augmentation distribution was the same as described in~\cref{sec:experiments}.
We approximated the upper bound term as follows:
\begin{align}
    \frac{1}{N}
    \sum_{i=1}^{N}
    \frac{1}{N_{y_i}}
    \sum_{j \neq i}^{N}
    \Ibb \left[ y_i = y_j \right]
    \times
    \left| \fbf(\abf_i(\xbf_i)) \cdot \left[
        \fbf(\abf_j(\xbf_j)) - \fbf(\abf_i^{+} (\xbf_i))
    \right] \right|,
    \label{eq:empirical-collision-bound}
\end{align}
where recall that $N_y$ is the number of samples whose label is $y$.
We reported the averaged value of \cref{eq:empirical-collision-bound} with respect to the random seeds on~\cref{tab:bounds}.

\subsection{Comprehensive Results of~\cref{tab:bounds}}

\cref{tab:cifar10-all-bound,tab:cifar100-all-bound} show the comprehensive results of~\cref{tab:bounds} for CIFAR-10 and CIFAR-100, respectively.
\cref{fig:comparison_bound_curve_cifar10} shows upper bounds of supervised loss and the linear accuracy on the validation dataset.

\begin{figure}
    \centering
    \begin{subfigure}[b]{0.47\textwidth}
        \centering
        \includegraphics[width=\textwidth]{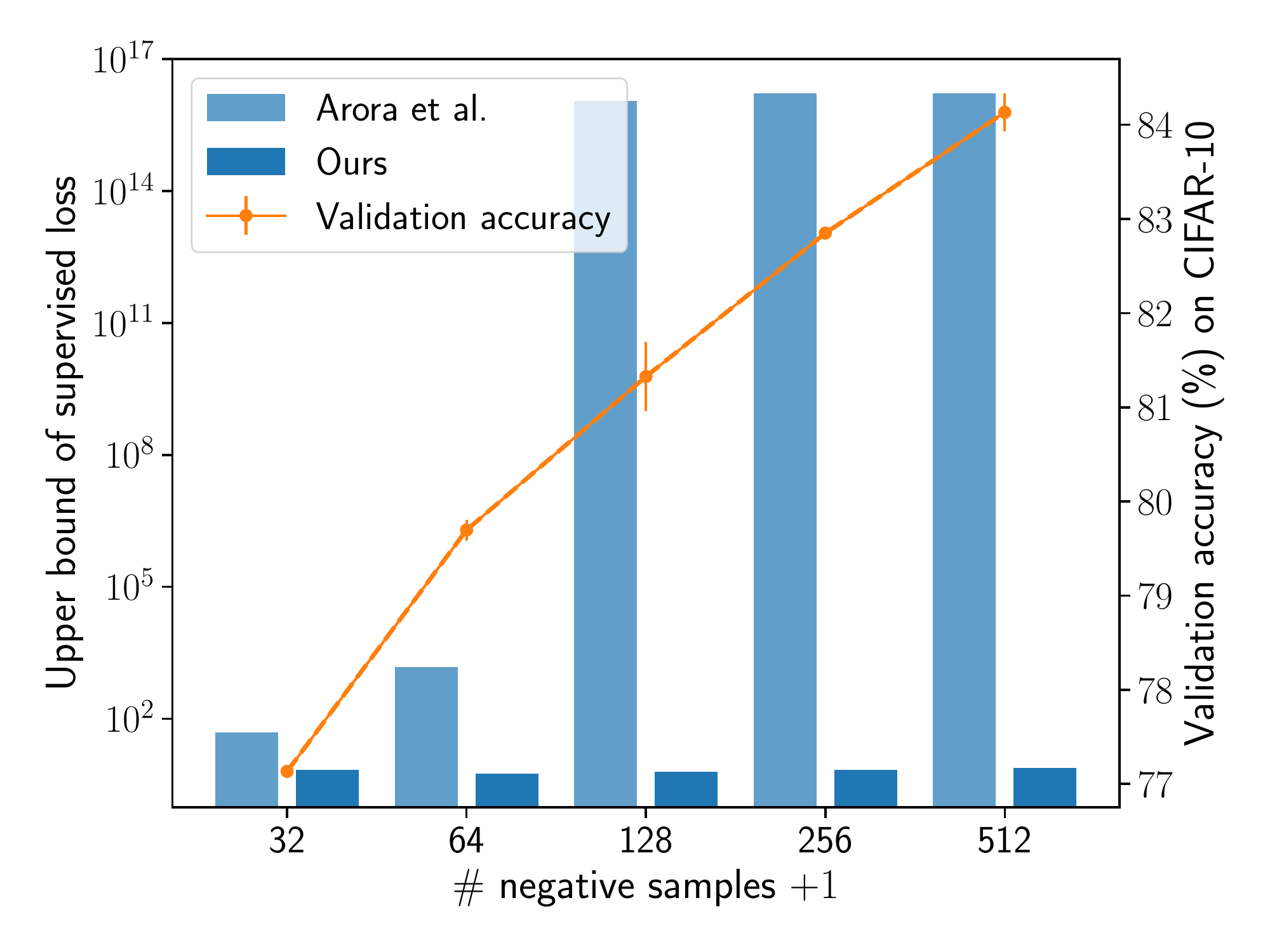}
        \caption{
            CIFAR-10.
        }
        \label{fig:comparison_bound_curve_cifar10}
    \end{subfigure}
    \begin{subfigure}[b]{0.47\textwidth}
        \centering
        \includegraphics[width=\textwidth]{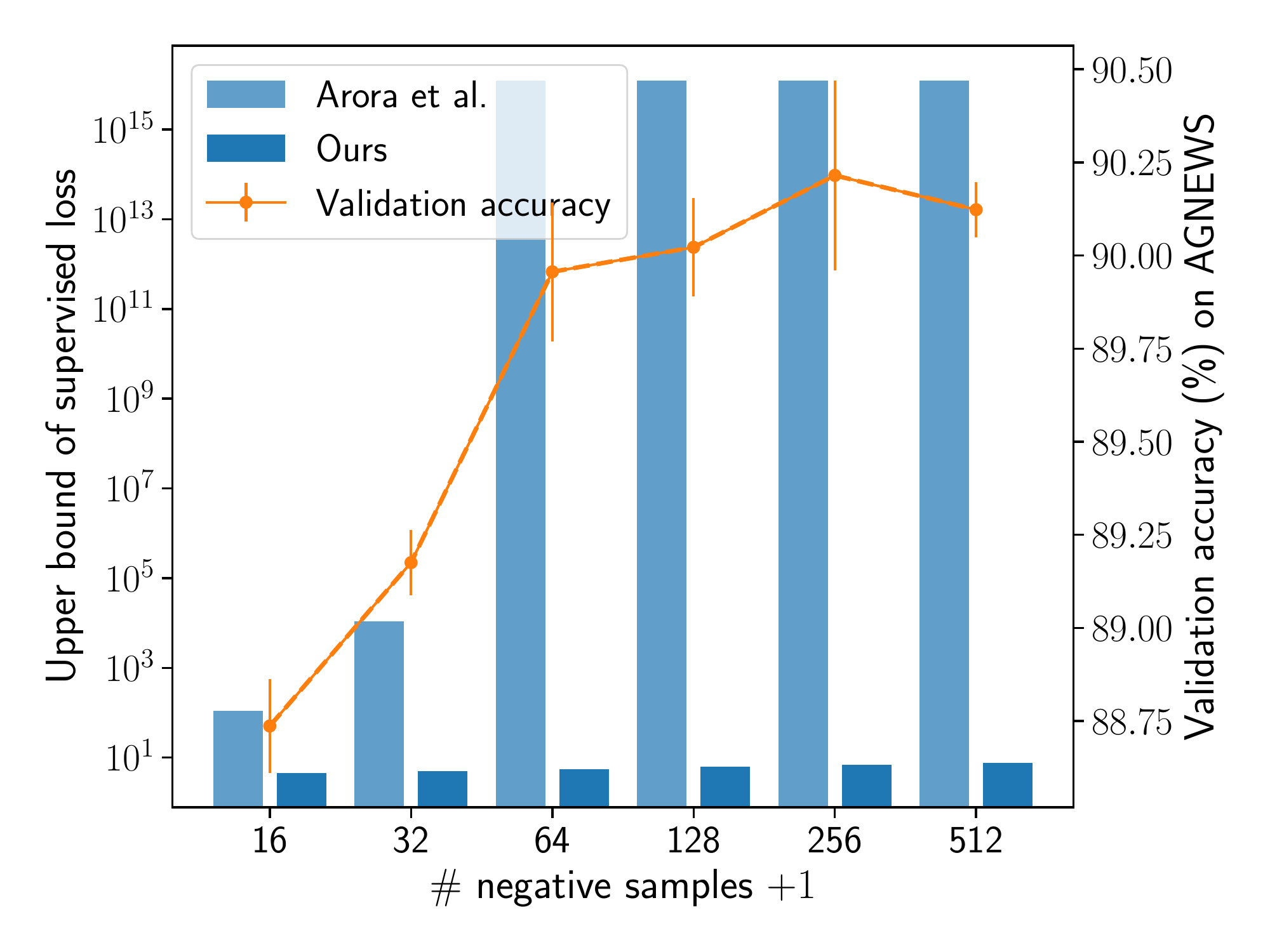}
        \caption{
            AGNews
        }
        \label{fig:comparison_bound_curve_ag_news}
    \end{subfigure}

    \caption{
            Upper bounds of supervised loss and validation accuracy.
    }
    \label{fig:additional-upper-bounds-plots}
\end{figure}

\begin{table}
    \centering
    \caption{
        The bound values on CIFAR-10 experiments with different $K+1$.
        CURL bound and its quantities are shown with $\dagger$.
        The proposed ones are shown without $\dagger$.
        Since the proposed collision values are half of $^\dagger$Collision, they are omitted.
        The reported values contain their coefficient except for Collision bound.
    }
    \label{tab:cifar10-all-bound}
    \begin{tabular}{llrrrrr}
\toprule
 $K+1$                         &                                &    $32$ &    $64$ &   $128$ &   $256$ &   $512$ \\
\midrule
 $\tau$                         &                                &  $0.96$ &  $1.00$ &  $1.00$ &  $1.00$ &  $1.00$ \\
 $\upsilon$                     &                                &  $0.69$ &  $0.99$ &  $1.00$ &  $1.00$ &  $1.00$ \\
 $\mu$ acc                      &                                & $72.75$ & $75.30$ & $77.22$ & $78.60$ & $80.12$ \\
 Linear acc                  &                                & $77.13$ & $79.70$ & $81.33$ & $82.85$ & $84.13$ \\
 Linear acc w/o              &                                & $82.02$ & $83.88$ & $85.43$ & $86.68$ & $87.66$ \\
 $\Lin$                       & \cref{eq:test-self-sup-loss}    &  $2.02$ &  $2.64$ &  $3.29$ &  $3.96$ &  $4.64$ \\
 $d(\mathbf{f})$                      & \cref{eq:mean-supervised}       & $-1.16$ & $-1.17$ & $-1.18$ & $-1.18$ & $-1.19$ \\
 $^{\dagger}\Lin$ bound          & \cref{eq:decompose-sup}         &  $0.23$ &  $0.76$ &  $1.41$ &  $2.08$ &  $2.75$ \\
 \hspace{2ex} $^{\dagger}$Collision &                                &  $1.32$ &  $1.93$ &  $2.58$ &  $3.26$ &  $3.94$ \\
 \hspace{2ex} $^{\dagger}L^{\mu}_{\mathrm{sup}}$    &                                &  $0.05$ &  $0.00$ &  $0.00$ &  $0.00$ &  $0.00$ \\
 \hspace{2ex} $^{\dagger}L^{\mu}_{\mathrm{sub}}$    &                                &  $0.01$ &  $0.00$ &  $0.00$ &  $0.00$ &  $0.00$ \\
 $\Lin$ bound                 & \cref{eq:proposed-bound}        &  $0.39$ &  $0.70$ &  $1.02$ &  $1.35$ &  $1.69$ \\
 \hspace{2ex} $L^{\mu}_{\mathrm{sup}}$           &                                &  $0.63$ &  $0.90$ &  $0.91$ &  $0.90$ &  $0.90$ \\
 \hspace{2ex} $L^{\mu}_{\mathrm{sub}}$           &                                &  $0.26$ &  $0.01$ &  $0.00$ &  $0.00$ &  $0.00$ \\
 Collision bound             & \cref{eq:upper-bound-collision} &  $0.60$ &  $0.61$ &  $0.61$ &  $0.62$ &  $0.62$ \\
\bottomrule
\end{tabular}
\end{table}

\begin{table}
    \centering
    \footnotesize
    \caption{
        The bound values on CIFAR-100 experiments with different $K+1$.
        CURL bound and its quantities are shown with $\dagger$.
        The proposed ones are shown without $\dagger$.
        Since the proposed collision values are half of $^\dagger$Collision, they are omitted.
        The reported values contain their coefficient except for Collision bound.
    }
    \label{tab:cifar100-all-bound}
    \begin{tabular}{llrrrrrrrr}
\toprule
 $K+1$                         &                                &   $128$ &   $256$ &   $384$ &   $512$ &   $640$ &   $768$ &   $896$ &   $1024$ \\
\midrule
 $\tau$                         &                                &  $0.72$ &  $0.92$ &  $0.98$ &  $0.99$ &  $1.00$ &  $1.00$ &  $1.00$ &   $1.00$ \\
 $\upsilon$                     &                                &  $0.00$ &  $0.00$ &  $0.15$ &  $0.62$ &  $0.90$ &  $0.98$ &  $1.00$ &   $1.00$ \\
 $\mu$ acc                      &                                & $32.74$ & $34.22$ & $35.27$ & $35.98$ & $36.58$ & $37.10$ & $36.84$ &  $37.50$ \\
 Linear acc                  &                                & $42.01$ & $43.59$ & $44.28$ & $45.09$ & $45.72$ & $46.06$ & $45.50$ &  $46.52$ \\
 Linear acc w/o              &                                & $57.92$ & $58.91$ & $59.51$ & $59.30$ & $59.35$ & $59.62$ & $59.11$ &  $59.46$ \\
 $\Lin$                       & \cref{eq:test-self-sup-loss}    &  $3.32$ &  $3.98$ &  $4.38$ &  $4.66$ &  $4.88$ &  $5.06$ &  $5.21$ &   $5.34$ \\
 $d(\mathbf{f})$                      & \cref{eq:mean-supervised}       & $-0.99$ & $-0.98$ & $-0.97$ & $-0.97$ & $-0.96$ & $-0.95$ & $-0.96$ &  $-0.95$ \\
 $^{\dagger}\Lin$ bound          & \cref{eq:decompose-sup}         &  $0.72$ &  $0.47$ &  $0.58$ &  $0.79$ &  $0.98$ &  $1.15$ &  $1.29$ &   $1.43$ \\
 \hspace{2ex} $^{\dagger}$Collision &                                &  $0.69$ &  $1.15$ &  $1.48$ &  $1.73$ &  $1.93$ &  $2.10$ &  $2.24$ &   $2.37$ \\
 \hspace{2ex} $^{\dagger}L^{\mu}_{\mathrm{sup}}$    &                                &  $0.00$ &  $0.00$ &  $0.01$ &  $0.01$ &  $0.01$ &  $0.00$ &  $0.00$ &   $0.00$ \\
 \hspace{2ex} $^{\dagger}L^{\mu}_{\mathrm{sub}}$    &                                &  $1.02$ &  $0.30$ &  $0.07$ &  $0.01$ &  $0.00$ &  $0.00$ &  $0.00$ &   $0.00$ \\
 $\Lin$ bound                 & \cref{eq:proposed-bound}        &  $1.18$ &  $1.53$ &  $1.72$ &  $1.86$ &  $1.96$ &  $2.05$ &  $2.12$ &   $2.19$ \\
 \hspace{2ex} $L^{\mu}_{\mathrm{sup}}$           &                                &  $0.00$ &  $0.00$ &  $0.30$ &  $1.21$ &  $1.76$ &  $1.92$ &  $1.95$ &   $1.95$ \\
 \hspace{2ex} $L^{\mu}_{\mathrm{sub}}$           &                                &  $1.82$ &  $1.93$ &  $1.66$ &  $0.75$ &  $0.20$ &  $0.04$ &  $0.01$ &   $0.00$ \\
 Collision bound             & \cref{eq:upper-bound-collision} &  $0.52$ &  $0.52$ &  $0.52$ &  $0.51$ &  $0.51$ &  $0.51$ &  $0.51$ &   $0.51$ \\
\bottomrule
\end{tabular}
\end{table}

\subsection{Experiments on Natural Language Processing}
\label{sec:nlp-experiments}

\citet{Arora2019ICML} conducted experiments for contrastive unsupervised sentence representation learning on Wiki-3029 dataset that contains $3\,029$ classes.
However, we cannot use this dataset to perform similar experiments to CIFAR-10/100.
This is because we need a huge number of negative samples to perform the same control experiments as CIFAR-10/100's experiments from the coupon collector’s problem.
Concretely, we need to more than $26\,030$ negative samples from \cref{eq:coupon-expectation}.
Such large negative samples cause an optimization issue in practice.
In addition, unlike self-supervised learning on vision,
self-supervised learning algorithms on text do not use a large number of negative samples in practice.
For example, \citet{Logeswaran2018ICLR,Gao2021EMNLP} use at most $399$ and $1\,023$ negative samples in their experiments respectively.%
\footnote{Precisely, all other samples in the same mini-batch are treated as negative samples.}

\paragraph{Dataset and Data Augmentation}
We used the AGNews classification dataset~\citep{Zhang2015NeurIPS} due to the aforementioned difficulty of the experiment on the Wiki-3029 dataset.
The dataset contains $4$ supervised classes and $120\,000$ training samples and $7\,600$ validation samples.
As pre-processing, we used a tokenizer of torchtext with its default option: \texttt{basic\_english}.
After that, we removed words whose frequency was less than $5$ in the training dataset.
As a data augmentation, we randomly delete $20\%$ words in each samples~\citep{Gao2021EMNLP}.
We tried a different data augmentation that randomly replaced $20\%$ words with one of the predefined similar words inspired by~\citet{Wang2015EMNLP}.
To obtain similar words, we used the five most similar words in pre-trained word vectors on the Common Crawl dataset~\citep{Mikolov2018LREC}.
If a word in the training data of the AGNews dataset did not exist in the pre-trained word vector's dictionary,
we did not replace the word.

\paragraph{Self-supervised Learning}
To compare the performance of supervised classification to the reported results on the AGNews dataset, we modify the supervised fastText model~\citep{Joulin2017EACL} to model a feature encoder $\fbf$.
The self-supervised model consists of a word embedding layer, an average pooling layer over the words, and the same nonlinear projection head as the CIFAR 10/100 experiment. The number of hidden units in the embedding layer and projection head was $50$.

We only describe the difference from the vision experiment because we mainly follow the vision experiment.
We trained the encoder by using PyTorch on a single GPU because the training was fast enough on a single GPU due to the model's simplicity.
The number of epochs was $100$.
We used linear warmup at each step during the first $10$ epochs.
We did not apply weight decay by following~\citet{Joulin2017EACL}.
For the base learning rate $\texttt{lr} \in \{ 1.0, 0.1 \}$ and initialization of learning rate,
either $\texttt{lr} \times \frac{K+1}{256}$ or $\texttt{lr} \times \sqrt{K+1}$, which are used in~\citet{Chen2020ICML}.

\paragraph{Linear Evaluation}
We followed the same optimization procedure as in the CIFAR 10/100 experiments except for the number of epochs that was $10$ and performing single GPU training.
We used the mean classifier's validation accuracy as a hyper-parameter selection criterion since we perform grid-search among two types of data augmentation, two learning rates, and two learning rate initialization methods.
Note that the deletion data augmentation performed better than the replacement one.

\paragraph{Bound Evaluation}
Same as in the CIFAR 10/100 experiments except for $K+1 \in \{32, 64, 128, 256, 512 \}$ since the number of classes is $4$.

\paragraph{Results}
\cref{tab:ag-all-bound} shows the quantities of bound-related values.
When $K$ was too small, both lower bounds were vacuous because an InfoNCE value should be non-negative.
However, the lower bounds were negative.
This vacuousness comes from $d(\fbf)$ that takes negative value.
\cref{fig:comparison_bound_curve_ag_news} shows the upper bound of mean classifier and linear accuracy on the validation dataset by rearranging InfoNCE's lower bounds.
\cref{fig:comparison_bound_curve_ag_news} shows the similar tendency to \cref{fig:comparison_bound_curve}; by increasing $K$, the existing bound explodes, but the proposed bound does not.

\begin{table}
    \centering
    \footnotesize
    \caption{
        The bound values on AGNews experiments with different $K+1$.
        CURL bound and its quantities are shown with $\dagger$.
        The proposed ones are shown without $\dagger$.
        Since the proposed collision values are half of $^\dagger$Collision, they are omitted.
        The reported values contain their coefficient except for Collision bound.
    }
    \label{tab:ag-all-bound}
    \begin{tabular}{llrrrrrr}
\toprule
 $K+1$                         &                                &    $16$ &    $32$ &    $64$ &   $128$ &   $256$ &   $512$ \\
\midrule
 $\tau$                         &                                &  $0.99$ &  $1.00$ &  $1.00$ &  $1.00$ &  $1.00$ &  $1.00$ \\
 $\upsilon$                     &                                &  $0.96$ &  $1.00$ &  $1.00$ &  $1.00$ &  $1.00$ &  $1.00$ \\
 $\mu$ acc                      &                                & $87.09$ & $88.41$ & $89.12$ & $89.38$ & $89.47$ & $89.54$ \\
 Linear acc                  &                                & $88.74$ & $89.18$ & $89.96$ & $90.02$ & $90.21$ & $90.12$ \\
 Linear acc w/o              &                                & $87.11$ & $87.94$ & $89.05$ & $89.49$ & $89.62$ & $89.45$ \\
 $L_{\mathrm{Info}}$                       & \cref{eq:test-self-sup-loss}    &  $1.23$ &  $1.77$ &  $2.37$ &  $3.02$ &  $3.69$ &  $4.37$ \\
 $d(\mathbf{f})$                      & \cref{eq:mean-supervised}       & $-1.72$ & $-1.77$ & $-1.81$ & $-1.82$ & $-1.83$ & $-1.84$ \\
 $^{\dagger}L_{\mathrm{Info}}$ bound          & \cref{eq:decompose-sup}         & $-0.22$ &  $0.35$ &  $0.99$ &  $1.65$ &  $2.34$ &  $3.02$ \\
 \hspace{2ex} $^{\dagger}$Collision &                                &  $1.49$ &  $2.13$ &  $2.80$ &  $3.48$ &  $4.16$ &  $4.85$ \\
 \hspace{2ex} $^{\dagger}L^{\mu}_{\mathrm{sup}}$    &                                &  $0.02$ &  $0.00$ &  $0.00$ &  $0.00$ &  $0.00$ &  $0.00$ \\
 \hspace{2ex} $^{\dagger}L^{\mu}_{\mathrm{sub}}$    &                                &  $0.00$ &  $0.00$ &  $0.00$ &  $0.00$ &  $0.00$ &  $0.00$ \\
 $L_{\mathrm{Info}}$ bound                 & \cref{eq:proposed-bound}        & $-0.39$ & $-0.10$ &  $0.22$ &  $0.55$ &  $0.89$ &  $1.23$ \\
 \hspace{2ex} $L^{\mu}_{\mathrm{sup}}$           &                                &  $0.57$ &  $0.61$ &  $0.63$ &  $0.64$ &  $0.64$ &  $0.65$ \\
 \hspace{2ex} $L^{\mu}_{\mathrm{sub}}$           &                                &  $0.02$ &  $0.00$ &  $0.00$ &  $0.00$ &  $0.00$ &  $0.00$ \\
 Collision bound             & \cref{eq:upper-bound-collision} &  $0.87$ &  $0.89$ &  $0.91$ &  $0.92$ &  $0.93$ &  $0.93$ \\
 $^{\dagger}$ $\ln L^{\mu}_{\mathrm{sup}}$ upper bound &                                &  $4.72$ &  $9.28$ & $37.06$ & $37.05$ & $37.04$ & $37.04$ \\
 $\ln L^{\mu}_{\mathrm{sup}}$ upper bound        &                                &  $1.52$ &  $1.60$ &  $1.72$ &  $1.83$ &  $1.93$ &  $2.02$ \\
\bottomrule
\end{tabular}
\end{table}

\end{document}